\documentclass[10pt]{article}

\usepackage[margin=1in]{geometry}
\usepackage[authoryear,round]{natbib}
\usepackage{parskip}
\usepackage[utf8]{inputenc}
\usepackage[T1]{fontenc}
\usepackage{lmodern,microtype,amsmath,amssymb,amsthm,mathtools}
\usepackage{algorithm,algpseudocode,booktabs,tabularx,placeins,needspace,flafter}
\usepackage[table,dvipsnames]{xcolor}
\definecolor{TableInk}{RGB}{40,62,82}
\definecolor{TableHead}{RGB}{232,239,245}
\definecolor{TableOurs}{RGB}{241,246,250}
\usepackage{xurl}
\usepackage[colorlinks,linkcolor=blue,citecolor=ForestGreen,urlcolor=blue]{hyperref}
\usepackage{etoc}
\usepackage{caption}
\usepackage[nameinlink,noabbrev]{cleveref}
\usepackage{nicefrac}
\allowdisplaybreaks
\ifdefined\pdfminorversion\pdfminorversion=7\fi
\floatstyle{ruled}
\restylefloat{algorithm}
\algrenewcommand\algorithmicrequire{\textbf{Input:}}
\algrenewcommand\algorithmicensure{\textbf{Output:}}

\newtheorem{theorem}{Theorem}
\newtheorem{assumption}{Assumption}
\newtheorem{lemma}{Lemma}[section]
\theoremstyle{remark}
\newtheorem{remark}{Remark}
\theoremstyle{plain}
\DeclareMathOperator*{\argmax}{arg\,max}
\newcommand{\E}{\mathbb E}
\newcommand{\Prob}{\mathbb P}
\newcommand{\one}{\mathbf 1}
\newcommand{\tO}{\widetilde{\mathcal O}}
\newcommand{\ind}[1]{\mathbf 1\{#1\}}
\newcommand{\norm}[1]{\lVert#1\rVert}
\newcommand{\cP}{\mathcal P}
\newcommand{\cC}{\mathcal C}
\newcommand{\cE}{\mathcal E}
\newcommand{\cH}{\mathcal H}
\newcommand{\cF}{\mathcal F}
\newcommand{\KL}{\operatorname{KL}}
\newcommand{\TV}{\operatorname{TV}}
\newcommand{\kl}{\operatorname{kl}}

\makeatletter
\providecommand{\theHALG@line}{\thealgorithm.\arabic{ALG@line}}
\renewcommand{\theHALG@line}{\thealgorithm.\arabic{ALG@line}}
\makeatother

\title{Toward Optimal Regret in Adversarial MDPs\\ with Stochastic Hard Constraints}
\author{%
\begin{tabular}{c@{\hspace{2em}}c}
Qian Zuo & Francesco Emanuele Stradi \\
{\small University of Edinburgh} & {\small Politecnico di Milano} \\
{\small\texttt{qian.zuo@ed.ac.uk}} & {\small\texttt{francescoemanuele.stradi@polimi.it}}
\end{tabular}%
}
\date{}
\hypersetup{
  pdftitle={Toward Optimal Regret in Adversarial MDPs with Stochastic Hard Constraints},
  pdfauthor={Qian Zuo and Francesco Emanuele Stradi}
}

\begin{document}
\raggedbottom
\etocdepthtag{main}
\maketitle

\begin{abstract}
We study episodic \emph{constrained Markov decision processes} with adversarial losses under stochastic \emph{hard} constraints. Specifically, starting from a known strictly feasible policy with margin $d$, we seek to obtain optimal regret while satisfying the expected cost constraints in every episode.
In this setting, \citet{stradi2025} show that a carefully designed mixing rule attains regret of order $\tO(\sqrt T/\min\{d,d^2\})$. Interestingly, they also provide a lower bound of order $\Omega(\sqrt T/\rho)$ for the same setting, where $\rho$ is the Slater margin of the offline problem and can be much larger than $d$. In this work, we build on their approach to obtain optimal regret dependence on these margins. 
Specifically, we propose MA-OPS, an algorithm that combines an optimistic search for the Slater margin with a pessimistic evaluation of the selected policies to safely learn a policy with a large feasibility margin. This policy is then used to minimize regret while satisfying the constraints at every episode.
In particular, we show that MA-OPS attains regret $\tO(\sqrt T/\rho+1/(d\rho))$. 
Finally, we provide a matching lower bound, showing that the dependence on $T$, $d$, $\rho$ in the regret bound is optimal up to logarithmic factors.
\end{abstract}

\section{Introduction}
Reinforcement learning has been successfully applied to many challenging problems, including robotic control, game playing, and aligning language models with human preferences \citep{kaufmann2023,silver2016,ouyang2022}. These problems are commonly modeled through Markov decision processes (MDPs), in which a learner sequentially interacts with an environment by selecting actions, observing state transitions, and incurring losses \citep{sutton2018}. In this paper, we consider episodic adversarial MDPs, where the transition kernel is fixed, while losses can change arbitrarily across episodes \citep{jin2020}. Thus, no statistical assumption on the losses is required, allowing the model to capture learning problems with changing objectives.

In many applications, the learner must also satisfy requirements on quantities such as resource consumption. Constrained MDPs (CMDPs, for short)~\citep{Altman1999ConstrainedMD} model these requirements by introducing costs whose expected cumulative values must remain below given thresholds. When such requirements must hold throughout the learning process, including during exploration, one needs algorithms that satisfy \emph{hard constraints}. Specifically, all constraints must be satisfied in every episode on a common high-probability event \citep{liu2021zero,stradi2025}. We study this problem in CMDPs with adversarial losses, unknown transitions, stochastic constraints, and bandit feedback.

\begin{table}[t]
\captionsetup{font=small}
\caption{Comparison of regret and cumulative positive constraint violation bounds under stochastic constraints and bandit feedback.
The first two rows concern constrained multi-armed bandits, and the remaining rows concern episodic CMDPs with unknown transitions.
Here, $d$ denotes the input policy safety margin, $\rho$ the Slater margin ($\rho\ge d$), and $T$ the number of episodes (rounds in bandits).
We show only the dependence on $T$, $d$, and $\rho$, omitting logarithmic factors.}
\label{tab:margin-comparison}
\centering
\small
\setlength{\tabcolsep}{2pt}
\renewcommand{\arraystretch}{1.13}
\renewcommand{\tabularxcolumn}[1]{m{#1}}
\begin{tabularx}{\textwidth}{lcccc>{\centering\arraybackslash}X}
\arrayrulecolor{TableInk}\toprule
\rowcolor{TableHead}
\textbf{Algorithm} & \textbf{Losses} & \shortstack{\textbf{Input} \textbf{margin}} & \textbf{Regret} & \textbf{Violation} & \textbf{Lower bound} \\
\midrule
OPB \citep{pacchiano2021} & Stoch. & $d$ & $\tO(\sqrt T/d)$ & $0$ & $\Omega(\max\{\sqrt T,1/d^2\})$ \\
SOLB \citep{genalti2025datadependent} & Adv. & $\rho$ & $\tO(\sqrt T/\rho)$ & $0$ & $\Omega(\sqrt T/\rho)$ \\
\midrule
OptPess-LP \citep{liu2021zero} & Stoch. & $d$ & $\tO(\sqrt T/d+1/\min\{d,d^2\})$ & $0$ & --- \\
DOPE+ \citep{yu2025improved} & Stoch. & $d$ & $\tO(\sqrt T/d+1/d^2)$ & $0$ & --- \\
\addlinespace[2pt]
S-OPS \citep{stradi2025} & Adv. & $d$ & $\tO(\sqrt T/\min\{d,d^2\})$ & $0$ & $\Omega(\sqrt T/\rho)$ \\
CV-OPS \citep{stradi2025} & Adv. & -- & $\tO(\sqrt T/\min\{\rho,\rho^2\}+1/\rho^4)$ & $\tO(1/\rho^4)$ & $\Omega(\sqrt T/\rho)$ \\
\addlinespace[3pt]
\rowcolor{TableOurs}
\textbf{S-OPS (Theorem~\ref{thm:sops})} & Adv. & $d$ & $\tO(\sqrt T/d)$ & $0$ & --- \\
\rowcolor{TableOurs}
\textbf{MA-OPS (Theorems~\ref{thm:combined},~\ref{thm:lower})} & Adv. & $d$ & $\tO(\sqrt T/\rho+1/(d\rho))$ & $0$ & $\Omega(\sqrt T/\rho+1/(d\rho))$ \\
\bottomrule
\end{tabularx}
\end{table}

A standard approach to ensure safety during exploration assumes that the learner is given a strictly feasible baseline policy together with upper bounds on its expected cumulative costs \citep{yu2025improved}. The learner can then mix this baseline with the algorithmically selected policies whose costs are not well estimated. The amount of exploration allowed by such an approach depends on the \emph{input policy safety margin} $d$, namely, the smallest margin by which the input policy satisfies all the constraints. When $d$ is small, ensuring safety may require selecting the baseline for many episodes, even when this policy incurs large losses.

However, the margin given as input to the learner may be much smaller than the largest one attainable in the CMDP. We denote the latter by the \emph{Slater margin} $\rho$, defined as the largest common slack with which any policy can satisfy all constraints, such that it holds
that $d\leq\rho$. 
For $d\leq1$, \citet{stradi2025} provide the Safe Optimistic Policy Search (S-OPS) algorithm, which attains $\tO(\sqrt T/d^2)$ regret over $T$ episodes, while establishing an $\Omega(\sqrt T/\rho)$ lower bound for the problem. Thus, these results leave open the question of the optimal joint dependence on $d$ and $\rho$ when constraints must be satisfied at every episode. Specifically, we focus on the following research question.
\begin{quote}
\centering
\emph{Is it possible to design an algorithm that attains optimal regret dependence on the initial input margin $d$ and the unknown Slater parameter $\rho$, while satisfying the constraints at every episode?}
\end{quote}
We answer this question affirmatively by introducing \emph{Margin-Adaptive Optimistic Policy Search} (MA-OPS). The algorithm first searches for a policy with a large safety margin while ensuring safety throughout this phase, and then uses this policy as the baseline for loss minimization. This allows us to attain a regret bound of order $\tO(\sqrt T/\rho+1/(d\rho))$, which we then prove to be optimal in $T,\rho,d$ up to logarithmic factors.

\subsection{Original Contribution}

Our contributions can be summarized as follows.
\begin{itemize}
\item \textbf{Linear $1/d$ dependence for S-OPS} We establish a $\tO(\sqrt T/d)$ regret bound for S-OPS with high probability (Theorem~\ref{thm:sops}). This improves the dependence from $1/d^2$ to $1/d$ for $d<1$ with respect to the bound provided in \citet{stradi2025}.
\item \textbf{Safe learning of the Slater margin} We develop MA-OPS, that is, the first algorithm capable of learning the Slater margin and the associated policy in constant time, while satisfying the constraints at each episode. MA-OPS then uses S-OPS instantiated with the aforementioned policy as %
subroutine to attain the final regret guarantees.
\item \textbf{Learning guarantees} Specifically, MA-OPS learns a margin of at least $\rho/2$ within $\tO(1/(d\rho))$ episodes, guaranteeing safety throughout (Theorem~\ref{thm:identification}). For every finite $T$, MA-OPS jointly guarantees $\tO(\sqrt T/\rho+1/(d\rho))$ regret and feasibility in every episode with high probability (Theorem~\ref{thm:combined}). 
\item \textbf{Matching lower bound} We prove an $\Omega(\sqrt T/\rho+1/(d\rho))$ regret lower bound for any safe algorithm (Theorem~\ref{thm:lower}). This matches the MA-OPS guarantees on $T,d,\rho$ up to logarithmic factors.
\end{itemize}
A summary of the performance of our algorithm compared to the state-of-the-art is provided in Table~\ref{tab:margin-comparison}.

\subsection{Related Work}\label{sec:related_work}

Online CMDPs have been widely studied in stochastic, adversarial, and non-stationary settings \citep{efroni2020exploration,qiu2020,stradi2024online,ding2023nonstationary}.
In stochastic CMDPs, \citet{efroni2020exploration} obtain $\widetilde{\mathcal O}(\sqrt T)$ regret and cumulative positive constraint violation using linear programming.
\citet{stradi2025strong} further achieve the same dependence on $T$ through policy optimization. For hard constraints, early work by \citet{liu2021zero} established $\widetilde{\mathcal O}(\sqrt T)$ regret for stochastic CMDPs with unknown transitions, using a known strictly feasible baseline to ensure feasibility in every episode.
Subsequent work improved the dependence on the number of states \citep{bura2022dope} and the episode length \citep{yu2025improved}.
Across these results, the leading regret term scales linearly with $1/d$, while the bounds of \citet{liu2021zero} and \citet{yu2025improved} include an additive $1/d^2$ term when the initial margin is at most one.
\citet{stradi2025} extended this line to adversarial losses and also studied learning a baseline when no feasible policy is initially known.
With a given baseline, they obtain $\widetilde{\mathcal O}(\sqrt T/d^2)$ regret for $d\le1$, omitting other problem parameters.
Without one, their algorithm learns a margin comparable to $\rho$ but permits violations during the search.
We study CMDPs with adversarial losses and stochastic hard constraints, starting from a known strictly feasible baseline.
We recover the linear $1/d$ dependence in the leading regret term and learn a baseline with margin comparable to $\rho$, while maintaining feasibility in every episode.

Another closely related topic is constrained bandits.
For constraints on individual actions, \citet{amani2019safe} obtain sharper regret bounds when the optimal action has a known positive safety margin.
This margin can vanish even when a strictly feasible baseline exists.
For per-round expected-cost constraints, \citet{pacchiano2021} and \citet{pacchiano2025} instead express stochastic regret bounds in terms of a known baseline's margin, with leading terms linear in $1/d$ and expected-regret lower bounds containing an inverse-square margin term.
For adversarial losses, \citet{genalti2025datadependent} obtain loss-dependent bounds.
Their main analysis assumes a known maximum-margin strategy and its expected costs, so the initial margin equals the Slater margin.
Our lower bound separates $d$ from $\rho$, quantifying the regret cost of limited initial information even when a larger margin is attainable.

Due to space limitations, we defer an additional discussion to Appendix~\ref{app:related}.

\section{Preliminaries}\label{sec:setting}
We study an episodic constrained Markov decision process (CMDP) \mbox{$\mathcal M=\langle X,A,P,\{\ell_t\}_{t=1}^T,\{G_t\}_{t=1}^T,\alpha\rangle$} with adversarial losses and stochastic constraint costs. The learner interacts with $\mathcal M$ for $T$ episodes of length $L$, subject to $m$ constraints. Here, $X$ and $A$ are finite state and action spaces. The transition kernel is $P:X\times A\times X\to[0,1]$, where $P(x'\mid x,a)$ denotes the probability of moving from state $x$ to state $x'$ after taking action $a$. For episode $t\in[T]$, $\ell_t\in[0,1]^{|X\times A|}$ is the loss vector and $G_t\in[0,1]^{|X\times A|\times m}$ is the constraint cost matrix. The vector $\alpha\in[0,L]^m$ contains the cost thresholds, with $\alpha_i$ denoting the threshold for constraint $i\in[m]$. We write $[n]=\{1,\ldots,n\}$ for any positive integer $n$.

At the beginning of each episode $t\in[T]$, an adversary chooses the loss vector $\ell_t$. The cost matrices $G_t$ are drawn i.i.d.\ from an unknown distribution $\mathcal G$. For constraint $i\in[m]$, we denote the $i$-th column of $G_t$ by $g_{t,i}\in[0,1]^{|X\times A|}$. At a state--action pair $(x,a)\in X\times A$, the loss is $\ell_t(x,a)$ and the cost for constraint $i$ is $g_{t,i}(x,a)$. We denote the expected cost vector by $g_i=\E[g_{t,i}]\in[0,1]^{|X\times A|}$ and the expected cost matrix by $G=[g_1,\ldots,g_m]\in[0,1]^{|X\times A|\times m}$. The learner knows $X,A,L$ and $\alpha$, whereas $P$ and $G$ are unknown.

We consider loop-free CMDPs, in which $X$ is partitioned into $L+1$ layers $X_0,\ldots,X_L$, with $X_0=\{x_0\}$ and $X_L=\{x_L\}$. Transitions occur only between consecutive layers, i.e., $P(x'\mid x,a)>0$ only if $x\in X_k$ and $x'\in X_{k+1}$ for some $k\in\{0,\ldots,L-1\}$. A finite-horizon CMDP can be represented in this form by duplicating its states across time steps.

A Markov policy $\pi:X\times A\to[0,1]$ defines a probability distribution over actions at each state, with $\pi(a\mid x)$ denoting the probability of choosing action $a$ in state $x$. Starting from $x_0$, the learner samples $a_k\sim\pi(\cdot\mid x_k)$ and the next state is drawn as $x_{k+1}\sim P(\cdot\mid x_k,a_k)$ for $k\in\{0,\ldots,L-1\}$. At the beginning of episode $t\in[T]$, the learner selects a policy $\pi_t$ using observations from previous episodes. Under bandit feedback, the learner observes the trajectory of visited state--action pairs $(x_k,a_k)$ for $k\in\{0,\ldots,L-1\}$, their losses $\ell_t(x_k,a_k)$, and their costs $g_{t,i}(x_k,a_k)$ for constraint $i\in[m]$.

Given $P$ and a policy $\pi$, we denote the induced occupancy measure by $q^{P,\pi}\in[0,1]^{|X\times A\times X|}$. For $x\in X_k$, $a\in A$, and $x'\in X_{k+1}$ with $k\in\{0,\ldots,L-1\}$, its entry gives the probability of visiting $(x,a,x')$ in one episode, i.e., $q^{P,\pi}(x,a,x')=\Prob\!\left[x_k=x,\,a_k=a,\,x_{k+1}=x'\mid P,\pi\right]$. For an occupancy measure $q$, the state--action marginal vector in $[0,1]^{|X\times A|}$ has entries $q(x,a)=\sum_{x'}q(x,a,x')$. Summing over actions gives the state marginal vector in $[0,1]^{|X|}$, with entries $q(x)=\sum_a q(x,a)$. Since each episode has $L$ steps, $\sum_{x,a}q(x,a)=L$. We denote the set of occupancy measures under $P$ by $\Delta(P)\subseteq[0,1]^{|X\times A\times X|}$. Each $q\in\Delta(P)$ induces a Markov policy $\pi^q$, with $\pi^q(a\mid x)=q(x,a)/q(x)$ at states where $q(x)>0$. %

The expected cumulative loss of policy $\pi$ in episode $t\in[T]$ is $\ell_t^\top q^{P,\pi}=\sum_{x,a}\ell_t(x,a)q^{P,\pi}(x,a)$. Similarly, its expected cumulative cost for constraint $i\in[m]$ is $\E_{P,\pi}\!\left[\sum_{k=0}^{L-1}g_i(x_k,a_k)\right]=g_i^\top q^{P,\pi}$. A policy is feasible if $g_i^\top q^{P,\pi}\leq\alpha_i$ for every constraint $i\in[m]$. For the policy $\pi_t$ selected in episode $t$, we write $q_t=q^{P,\pi_t}$.

\subsection{Online CMDPs with Hard Constraints}
We measure the learner's performance in terms of regret against the best feasible policy in hindsight. Specifically, let $q^\star\in\Delta(P)$ be a minimizer of $\sum_{t=1}^T\ell_t^\top q$ subject to $G^\top q\leq\alpha$. We define the \emph{cumulative regret} as:
\begin{equation}
R_T=\sum_{t=1}^T\ell_t^\top(q_t-q^\star).
\label{eq:regret}
\end{equation}
Our goal is to achieve a regret bound that is sublinear in $T$, i.e., $R_T=o(T)$, while ensuring that $g_i^\top q_t\leq\alpha_i$ for every episode $t\in[T]$ and every constraint $i\in[m]$ with high probability. To satisfy the constraints from the first episode, we assume that the learner knows a \emph{potentially suboptimal} strictly feasible policy \citep{liu2021zero,bura2022dope,stradi2025}.
\begin{assumption}[Strictly feasible policy]\label{ass:baseline}
The learner knows a Markov policy $\pi^\diamond$ and a vector $\beta\in[0,L]^m$ satisfying $G^\top q^{P,\pi^\diamond}\leq\beta<\alpha$.
\end{assumption}

The upper bounds $\beta$ on the expected cumulative costs of $\pi^\diamond$ determine the \emph{initial margin} $d$. In contrast, the \emph{Slater margin} $\rho$ is the largest slack that a policy can achieve across the constraints. Formally,
\begin{equation*}
    d=\min_{i\in[m]}(\alpha_i-\beta_i), \quad
\rho=\max_{q\in\Delta(P)}\min_{i\in[m]}(\alpha_i-g_i^\top q).
\end{equation*}
The initial margin $d$ is known to the learner, whereas the Slater margin $\rho$ is unknown. Assumption~\ref{ass:baseline} implies Slater's condition \citep{boyd2004convex} and $0<d\leq\rho\leq L$. We use $\mathcal O$ and $\Omega$ to suppress universal constants, and $\widetilde{\mathcal O}$ to additionally suppress logarithmic factors.

\section{Improving The Margin Dependence of S-OPS}\label{sec:direct_improvement}
When the initial margin $d$ is small, safe exploration can require frequent use of the baseline, limiting how often the policy selected for exploration is executed. Our key observation is that this execution probability also weights the policy's constraint violation bound in the regret analysis. This weight links the regret incurred by using the baseline to uncertainty along observed trajectories. Using this relation, we show that S-OPS~\citep{stradi2025} achieves a regret bound with linear dependence on $1/d$ while satisfying the constraints in every episode.

\subsection{Safe Optimistic Policy Search (S-OPS)}\label{sec:sops-algorithm}
S-OPS~\citep{stradi2025} searches for low-loss policies under optimistic constraints and mixes them with the baseline $\pi^\diamond$ using pessimistic cost bounds. In episode $t\in[T]$, S-OPS maintains an occupancy measure $\widehat q_t\in\Delta(\cP_{t-1})$ using online mirror descent, where $\Delta(\cP_{t-1})$ is the set of occupancy measures induced by kernels in the transition confidence set $\cP_{t-1}$. We denote the empirical cost vectors and their confidence widths by $\widehat g_{t-1,i},\xi_{t-1}\in[0,1]^{|X\times A|}$, respectively. The mirror descent update imposes $(\widehat g_{t-1,i}-\xi_{t-1})^\top\widehat q_t\leq\alpha_i$ for each constraint $i\in[m]$. On the confidence event, the resulting optimistic set contains the true occupancy of every feasible policy. However, the induced policy $\widehat\pi_t=\pi^{\widehat q_t}$ may violate the true constraints.

To choose a mixing probability that ensures safety, S-OPS first computes upper bounds on the expected costs of $\widehat\pi_t$.
To account for transition uncertainty, it uses $\widehat u_t\in[0,1]^{|X\times A|}$, where $\widehat u_t(x,a)=\max_{P'\in\cP_{t-1}}q^{P',\widehat\pi_t}(x,a)$.
S-OPS then combines these occupancy bounds with the pessimistic cost estimates $\widehat g_{t-1,i}+\xi_{t-1}$ to obtain $w_t\in[0,L]^m$, where $w_{t,i}=\min\{L,(\widehat g_{t-1,i}+\xi_{t-1})^\top\widehat u_t\}$.
At the beginning of each episode, the randomized policy $\pi_t$ selects the baseline $\pi^\diamond$ with probability:
\begin{equation}
\lambda_{t-1}=\max_{i:\,w_{t,i}>\alpha_i}
\frac{w_{t,i}-\alpha_i}{w_{t,i}-\beta_i},
\label{eq:main-sops-mixing}
\end{equation}
and $\widehat\pi_t$ with probability $1-\lambda_{t-1}$, then follows the selected policy for all $L$ steps.
We set $\lambda_{t-1}=0$ when $w_t\leq\alpha$.
This randomization gives $q^{P',\pi_t}=\lambda_{t-1}q^{P',\pi^\diamond}+(1-\lambda_{t-1})q^{P',\widehat\pi_t}$ for every kernel $P'$.
With high probability, the cost bounds and the mixing rule therefore ensure $G^\top q_t\leq\lambda_{t-1}\beta+(1-\lambda_{t-1})w_t\leq\alpha$.
Moreover, the probability of executing $\widehat\pi_t$ satisfies $1-\lambda_{t-1}>0$, allowing S-OPS to collect observations under this policy before it is known to be feasible.

After executing $\pi_t$, S-OPS uses the observed losses to construct the optimistic loss estimator $\widehat\ell_t\in[0,1/\gamma]^{|X\times A|}$, with implicit exploration parameter \mbox{$\gamma>0$}:
\begin{equation}
\widehat\ell_t(x,a)=\frac{\ell_t(x,a)I_t(x,a)}{u_t(x,a)+\gamma},
\label{eq:main-sops-estimator}
\end{equation}
where $u_t(x,a)=\max_{P'\in\cP_{t-1}}q^{P',\pi_t}(x,a)$ is the largest probability of visiting $(x,a)$ under $\pi_t$, and $I_t(x,a)$ is the indicator function which is $1$ when $(x,a)$ has been visited at episode $t$.
S-OPS also updates its cost and transition estimates after every episode.
A multiplicative update with $\widehat\ell_t$ and learning rate $\eta>0$, followed by a relative-entropy projection onto the updated optimistic set, gives $\widehat q_{t+1}$.
We provide the pseudocode of S-OPS in Algorithm~\ref{alg:sops} (Appendix~\ref{app:sops-specification}).

\subsection{Regret Guarantee}\label{sec:sops-guarantee}
We then obtain the following regret guarantee.
\begin{theorem}[Regret and constraint violation]\label{thm:sops}
Suppose Assumption~\ref{ass:baseline} holds. Let $\delta\in(0,1)$. S-OPS with \mbox{$\eta=\gamma=\sqrt{\frac{L\ln(8m|X|^2|A|T/\delta)}{T|X||A|}}$} guarantees, with probability at least $1-\delta$, $G^\top q_t\leq\alpha$ for every $t\in[T]$ and:
\[
R_T \leq \widetilde{\mathcal O}\!\left(
\frac{L^2|X|\sqrt{|A|T}}{d}
+\frac{L|X|^3|A|}{d}
\right).
\]
\end{theorem}

Theorem~\ref{thm:sops} guarantees sublinear regret with linear dependence on the inverse baseline margin $1/d$, while satisfying every constraint in any episode. The original S-OPS analysis assumes that the baseline policy’s expected costs are known exactly and gives regret $\tO(L^3|X|\sqrt{|A|T}\allowbreak/\min\{d,d^2\})$ \citep{stradi2025}. Our analysis improves the dependence on $1/d$ from quadratic to linear when $d<1$, and reduces the dependence on $L$ in the leading term. %
In stochastic CMDPs, the DOPE+ regret bound has a leading term linear in $1/d$ and an additive term quadratic in this inverse margin \citep{yu2025improved}. Theorem~\ref{thm:sops} shows that both regret terms depend only linearly on $1/d$ for adversarial losses.

\begin{remark}[Margin normalization]\label{rem:margin-normalization}
To compare our bounds with MAB ones where the costs are in $[0,1]$, we divide cumulative episode costs and constraint thresholds by $L$. The resulting margins are $d/L,\rho/L\in(0,1]$. Notice that the inverse normalized Slater margin satisfies $L/\rho\geq1$, even when $\rho>1$, and, thus, $(L/\rho)^2\geq L/\rho$. Therefore, comparing linear and quadratic dependence on the inverse Slater margin requires accounting for the dependence on the horizon $L$ as well.
\end{remark}

\subsection{Proof Sketch}\label{sec:sops-proof}
In the regret decomposition of \citet{stradi2025}, the regret incurred by using the baseline is at most $L\sum_{t=1}^T\lambda_{t-1}$. Their bound is quadratic in $1/d$ when $d<1$. We improve this bound by incorporating the execution probability of $\widehat\pi_t$ into the analysis.

For $t\in\{2,\ldots,T\}$, let $e_t$ be the largest estimated constraint violation, i.e., $e_t=\max_{i\in[m]}\{w_{t,i}-\alpha_i,0\}$. If $\lambda_{t-1}>0$, let $i_t$ attain the maximum in Equation~\eqref{eq:main-sops-mixing}. We rearrange the mixing rule to obtain:
\begin{equation}
\lambda_{t-1}(\alpha_{i_t}-\beta_{i_t})
=(1-\lambda_{t-1})(w_{t,i_t}-\alpha_{i_t}).
\label{eq:main-mixing-identity}
\end{equation}
Since $\alpha_{i_t}-\beta_{i_t}\geq d$, we obtain $\lambda_{t-1}\leq(1-\lambda_{t-1})e_t/d$. It therefore suffices to bound $\sum_{t=2}^{T}(1-\lambda_{t-1})e_t$.

Let $\cE_T$ be the event on which all cost and transition confidence bounds hold. On this event, the optimistic constraints bound $e_t$ by cost and transition estimation errors. For the cost error, we have $(1-\lambda_{t-1})\xi_{t-1}^\top q^{P,\widehat\pi_t}\leq\xi_{t-1}^\top q_t$. The right-hand side is the conditional expected sum of confidence widths along the observed trajectory.
For the transition error, we compare $\widehat q_t$ and $\widehat u_t$ with $q^{P,\widehat\pi_t}$ and bound their deviations by occupancy-weighted sums of confidence widths and their squares. These sums control the leading and lower-order terms, respectively. Weighting the bounds for $\widehat\pi_t$ and $\pi^\diamond$ by their execution probabilities also controls $\norm{u_t-q_t}_1$. Combining these estimates with the cost-error bound gives the following lemma.

\begin{lemma}\label{lem:main-weighted-errors}
Under the conditions of Theorem~\ref{thm:sops}, with probability at least $1-\delta/2$, the event $\cE_T$ holds and $\sum_{t=2}^T(1-\lambda_{t-1})e_t+\sum_{t=1}^T\norm{u_t-q_t}_1\leq\tO\!\left(L|X|\sqrt{|A|T}+|X|^3|A|\right)$.
\end{lemma}

Lemma~\ref{lem:main-weighted-errors} controls both the weighted constraint-violation bounds and the cumulative upper-occupancy error. Equation~\eqref{eq:main-mixing-identity} then bounds regret from baseline use by $L+(L/d)\sum_{t=2}^T(1-\lambda_{t-1})e_t$, where $L$ accounts for the initial episode. Retaining the execution weights gives linear dependence on $1/d$. The earlier analysis bounds unweighted errors using $1-\lambda_{t-1}\geq d/L$, introducing an additional factor $L/d$ when $d<1$.

We next control the mirror descent and loss-estimation terms. The estimated occupancy measure $\widehat q_t$ satisfies $(1-\lambda_{t-1})\widehat q_t=q^{P',\pi_t}-\lambda_{t-1}q^{P',\pi^\diamond}\leq u_t$ for some $P'\in\cP_{t-1}$. Together with $1-\lambda_{t-1}\geq d/L$, this implies $\widehat q_t/(u_t+\gamma)\leq L/d$. For the quadratic term in the mirror descent bound, we use $\widehat q_t(x,a)\widehat\ell_t(x,a)^2\leq(L/d)\widehat\ell_t(x,a)$. Implicit exploration bounds $\sum_{t,x,a}\widehat\ell_t(x,a)$ without further dependence on $d$ \citep{jin2020}. The ratio bound also controls the cumulative conditional bias by $(L/d)(\sum_t\norm{u_t-q_t}_1+\gamma|X||A|T)$, where Lemma~\ref{lem:main-weighted-errors} bounds the occupancy-error sum. Combining these bounds with standard concentration bounds for the remaining loss-estimation terms yields Theorem~\ref{thm:sops}. The full proofs are deferred to Appendix~\ref{app:sops}.
\section{MA-OPS: How to Safely Learn the Slater Margin}\label{sec:algorithms}
In this section, we introduce \emph{Margin-Adaptive Optimistic Policy Search} (MA-OPS), a meta-algorithm that safely learns a baseline with margin comparable to the unknown Slater margin $\rho$. %
By Theorem~\ref{thm:sops}, such a baseline makes subsequent regret scale with $1/\rho$. The main challenge is that $\rho$ is unknown, while policies considered during the search may violate the constraints. MA-OPS addresses this challenge by combining an optimistic margin search with pessimistic policy evaluation. Optimistic search selects a policy and estimates the largest attainable margin $\rho$. Pessimistic evaluation accounts for cost and transition uncertainty to obtain a lower bound on the selected policy's margin. The search ends once either the initial margin $d$ or the selected policy's pessimistic margin estimate is at least half the optimistic estimate of $\rho$. 

We present MA-OPS in Algorithm~\ref{alg:identification}.

\begin{algorithm}[t]
\caption{MA-OPS}\label{alg:identification}
\small
\algrenewcommand{\algorithmicindent}{1.05em}
\algrenewcommand{\algorithmiccomment}[1]{\hfill\textcolor{blue}{$\triangleright$~#1}}
\begin{algorithmic}[1]
\setlength{\itemsep}{1pt}
\Require $X,A,L,\alpha,T,\delta$ and $(\pi^\diamond,\beta)$ satisfying Assumption~\ref{ass:baseline}
\State $d\gets\min_{i\in[m]}(\alpha_i-\beta_i)$, $\delta_0\gets\delta/2$, $t\gets0$, $n\gets1$. Initialize all visit counts to zero
\State $(\widehat\pi^\diamond,\widehat\beta)\gets(\pi^\diamond,\beta)$
\State Set $\widehat\pi_1$ to the uniform policy and $w_{1,i}\gets L$ for $i\in[m]$
\State $\lambda_1\gets\max_{i\in[m]} (L-\alpha_i)/(L-\beta_i)$
\While{$t<T$ and $2d<L$}
  \State $t\gets t+1$ 
  \State Draw $\zeta_t\sim\operatorname{Bernoulli}(1-\lambda_n)$
  \State Execute $\begin{cases}
      \pi^\diamond, & \zeta_t=0,\\
      \widehat\pi_n, & \zeta_t=1
    \end{cases}$\label{line:id-execute}
  \State Observe the trajectory and feedback
  \State \textbf{if} $\zeta_t=0$ \textbf{then continue}\label{line:id-skip}
  \State $n\gets n+1$
  \State Update $\widehat G_{n-1},\Xi_{n-1},\cP_{n-1}$\label{line:id-update}
  \State Choose $\widehat q_n$ by \eqref{eq:optimistic-margin-search}\label{line:id-search}\Comment{Optimistic estimate of $\rho$}
  \State $\overline\rho_n\gets\min\{L,\min_{i\in[m]}[\alpha_i-(\widehat g_{n-1,i}-\xi_{n-1})^\top\widehat q_n]\}$\label{line:id-margin-value}
  \State \textbf{if} $2d\geq\overline\rho_n$ \textbf{then break}\label{line:id-original-stop}
  \State $\widehat\pi_n\gets\pi^{\widehat q_n}$
  \State Compute $w_n$ by \eqref{eq:id-cost} and set $\underline\rho_n\gets\min_{i\in[m]}(\alpha_i-w_{n,i})$\label{line:id-cost}\Comment{Pessimistic policy evaluation}
  \If{$2\underline\rho_n\geq\overline\rho_n$}\label{line:id-stop}\Comment{Adaptive stopping rule}
    \State $(\widehat\pi^\diamond,\widehat\beta)\gets(\widehat\pi_n,w_n)$
    \State \textbf{break}
  \EndIf
  \State $\lambda_n\gets
\max\limits_{\mathclap{i\in[m]:\,w_{n,i}>\alpha_i}}
(w_{n,i}-\alpha_i)/(w_{n,i}-\beta_i)$
\label{line:id-mix}\Comment{Safe exploration}
\EndWhile
\If{$t<T$}\label{line:sops-restart}
  \State Apply S-OPS with $(\widehat\pi^\diamond,\widehat\beta)$, horizon $T-t$, confidence $\delta_0$, and parameters from Theorem~\ref{thm:sops}\label{line:sops-call}
\EndIf
\end{algorithmic}
\end{algorithm}

\paragraph{Optimistic margin search}
We first notice that an underestimate of $\rho$ can cause the search to end with a baseline whose margin is smaller than $\rho/2$. MA-OPS therefore estimates $\rho$ optimistically while searching for a large-margin policy. Let $t$ denote the number of completed episodes and $n-1$ the number of exploration trajectories collected. For each constraint $i\in[m]$, MA-OPS computes the empirical cost vector $\widehat g_{n-1,i}\in[0,1]^{|X\times A|}$ from these trajectories. The estimates are combined with cost confidence widths $\xi_{n-1}\in[0,1]^{|X\times A|}$ to form the optimistic cost estimates $\widehat g_{n-1,i}-\xi_{n-1}$. MA-OPS then computes $\widehat q_n$ by solving the optimistic max--min problem:
\begin{equation}
\begin{multlined}
\hspace{-0.1cm}\argmax_{q\in\Delta(\cP_{n-1})}
\min\Bigl\{L,\min_{i\in[m]}\bigl[\alpha_i-(\widehat g_{n-1,i}-\xi_{n-1})^\top q\bigr]\Bigr\},
\end{multlined}
\label{eq:optimistic-margin-search}
\end{equation}
where $\cP_{n-1}$ is the transition confidence set constructed from the same trajectories. Let $\overline\rho_n$ be the optimal value and $\widehat\pi_n=\pi^{\widehat q_n}$ the induced policy. With high probability, $P\in\cP_{n-1}$ and $|\widehat g_{n-1,i}-g_i|\leq\xi_{n-1}$ hold at every episode for each constraint $i\in[m]$. On this event, an occupancy measure attaining $\rho$ belongs to $\Delta(\cP_{n-1})$, and replacing costs by their lower bounds cannot decrease its constraint slack. Since $\rho\leq L$, taking the minimum with $L$ preserves $\overline\rho_n\geq\rho$. The equivalent linear program is given in Appendix~\ref{app:occupancy}.

\paragraph{Policy evaluation and stopping}
Optimistic search uses cost lower bounds and a transition kernel that may differ from $P$, so it can overestimate the margin of $\widehat\pi_n$. MA-OPS therefore maximizes each expected cost over $\cP_{n-1}$ using the pessimistic cost estimates $\widehat g_{n-1,i}+\xi_{n-1}$. This gives $w_n\in[0,L]^m$, where:
\begin{equation}
w_{n,i}=\min\!\left\{L,\max_{P'\in\cP_{n-1}}
(\widehat g_{n-1,i}+\xi_{n-1})^\top q^{P',\widehat\pi_n}\right\}
\label{eq:id-cost}
\end{equation}
for each constraint $i\in[m]$. Each maximization evaluates the entire expected cost under a single transition kernel. The gap between the margin estimates is then controlled by cost and transition estimation error.

A strictly feasible policy can still have a margin much smaller than $\rho$. MA-OPS thus compares both the initial margin $d$ and the lower bound $\underline\rho_n:=\min_{i\in[m]}(\alpha_i-w_{n,i})$ with $\overline\rho_n/2$. We denote the occupancy measure of $\widehat\pi_n$ under $P$ by $q_n:=q^{P,\widehat\pi_n}\in\Delta(P)$. On the confidence event, it holds that $G^\top q_n\leq w_n$, and hence:
\begin{equation}
\underline\rho_n\leq\min_{i\in[m]}(\alpha_i-g_i^\top q_n)
\leq\rho\leq\overline\rho_n.
\label{eq:main-margin-sandwich}
\end{equation}
The search ends with the initial baseline when $2d\ge\overline\rho_n$. Otherwise, it ends with $\widehat\pi_n$ when $2\underline\rho_n\ge\overline\rho_n$. In either case, the selected baseline has a margin of at least $\rho/2$. Testing the initial baseline first avoids unnecessary policy evaluation and ensures that any replacement improves the known margin, since $d<\overline\rho_n/2\leq\underline\rho_n$. Since $\rho \leq L$, the search can be omitted when $2d\geq L$.

\paragraph{Safe exploration}
The cost upper bounds $w_n$ also determine how often $\widehat\pi_n$ can be executed while satisfying the constraints. We compute the baseline selection probability $\lambda_n\in[0,1)$ using the S-OPS mixing rule~\eqref{eq:main-sops-mixing} with $w_n$ (Line~\ref{line:id-mix}). We set $\lambda_n=0$ when $w_{n,i}\leq\alpha_i$ for every constraint $i\in[m]$. In each episode, we select $\pi^\diamond$ with probability $\lambda_n$ and $\widehat\pi_n$ with probability $1-\lambda_n$. We then execute the selected policy for all $L$ steps and collect the resulting trajectory and cost observations. With high probability, the randomized policy satisfies $G^\top q_t\leq\lambda_n\beta+(1-\lambda_n)w_n\leq\alpha$ in every episode of the search, even when $\widehat\pi_n$ is infeasible.

After executing $\widehat\pi_n$ ($\zeta_t=1$), MA-OPS updates the estimates (Line~\ref{line:id-update}) and repeats the optimistic search. During episodes in which $\pi^\diamond$ is executed ($\zeta_t=0$), the estimates, $n$, $\widehat\pi_n$, and $\lambda_n$ remain unchanged. Consequently, the next exploration trajectory is generated by $\widehat\pi_n$ under $P$, regardless of how many baseline episodes precede it. The empirical cost and confidence-width matrices $\widehat G_{n-1},\Xi_{n-1}\in[0,1]^{|X\times A|\times m}$ have $i$-th columns $\widehat g_{n-1,i}$ and $\xi_{n-1}$, respectively. The detailed definitions are given in Appendix~\ref{app:estimation}. Initially, $\widehat\pi_1$ selects actions uniformly, $w_{1,i}=L$ for every $i\in[m]$, and $\overline\rho_1=L$.

For the remaining $T-t$ episodes, we execute S-OPS with the selected baseline $\widehat\pi^\diamond$ and its cost upper bounds $\widehat\beta\in[0,L]^m$. Both phases use confidence parameter $\delta_0=\delta/2$. The learning rate $\eta$ and implicit-exploration parameter $\gamma$ are chosen as in Theorem~\ref{thm:sops} for horizon $T-t$ and confidence parameter $\delta_0$.

\section{Theoretical Guarantees}\label{sec:guarantees}
In this section, we establish the regret and constraint violation guarantees for MA-OPS. The full proofs are given in
Appendices~\ref{app:identification}--\ref{app:combined}.

\subsection{Main Result}\label{sec:main-result}
We first provide the theoretical guarantees of MA-OPS.
\begin{theorem}[Regret and constraint violation]\label{thm:combined}
Suppose Assumption~\ref{ass:baseline} holds. Let $\delta\in(0,1)$. With probability at least $1-\delta$, Algorithm~\ref{alg:identification} guarantees $G^\top q_t\leq\alpha$ for every $t\in[T]$ and regret:
\[
R_T\leq\tO\!\left(
\frac{L^2|X|\sqrt{|A|T}}{\rho}
+\frac{L^2|X|^3|A|}{d\rho}
\right).
\]
\end{theorem}
Theorem~\ref{thm:combined} guarantees sublinear regret with a leading term proportional to $1/\rho$ and zero constraint violation, without requiring prior knowledge of $\rho$. The bound separates learning with a large-margin baseline from the cost of searching for one safely. The leading term follows from Theorem~\ref{thm:sops} with a known margin of at least $\rho/2$. The additive term captures the cost of searching safely from the initial margin $d$ and is independent of $T$, including its logarithmic factors.

\begin{remark}%
Compared with S-OPS in Theorem~\ref{thm:sops}, MA-OPS reduces the leading regret term from $\tO(L^2|X|\sqrt{|A|T}/d)$ to $\tO(L^2|X|\sqrt{|A|T}/\rho)$.
For stochastic CMDPs, OptPess-LP~\citep{liu2021zero} and DOPE+~\citep{yu2025improved} guarantee zero constraint violation given a known strictly feasible baseline.
With one constraint and $S$ states before time expansion, DOPE+ attains $\tO(L^{5/2}S\sqrt{|A|T}/d+L^5S^3|A|/d^2)$ regret.
Under $|X|=\Theta(LS)$, our bound is $\tO(L^3S\sqrt{|A|T}/\rho+L^5S^3|A|/(d\rho))$ for adversarial losses, with a leading term independent of $m$ up to logarithmic factors.
In constrained bandits, SOLB~\citep{genalti2025datadependent} gives loss-dependent guarantees and worst-case regret $\tO(|A|\sqrt T/\rho)$ for $T\geq12|A|/\rho$, assuming a known maximum-margin policy and its expected costs in the main analysis.
CV-OPS~\citep{stradi2025} starts without a known feasible policy and learns a baseline with margin at least $\rho/2$, allowing constraint violation during the search.
Starting from a baseline with margin $d$, MA-OPS identifies such a baseline under unknown transitions within $\tO(L^2|X|^2|A|/(d\rho))$ episodes with zero constraint violation throughout (Theorem~\ref{thm:identification}).
\end{remark}

\subsection{Proof Analysis}\label{sec:identification-guarantee}
In this section, we provide the main intuition on the proofs which lead to Theorem~\ref{thm:combined}.

We first bound the number of episodes required for safe baseline
search. The key step is to bound the cumulative gap between the
optimistic and pessimistic margins. Together with the stopping conditions in Algorithm~\ref{alg:identification}, this bound controls the number of exploration trajectories.
The mixing rule then bounds the number of baseline executions
in terms of the same gap. Finally, we combine the resulting
search regret bound with Theorem~\ref{thm:sops} to prove
Theorem~\ref{thm:combined}.

Let $\tau$ denote the number of episodes until baseline search terminates and $N$ the number of exploration trajectories collected. We have the following theorem.
\begin{theorem}[Safe baseline search]\label{thm:identification}
Suppose Assumption~\ref{ass:baseline} holds and let $\delta\in(0,1)$. With probability at least $1-\delta$, the baseline search terminates after at most $\tO(L^2|X|^2|A|/(d\rho))$ episodes and collects at most $\tO(L^2|X|^2|A|/\rho^2)$ exploration trajectories. On the same event, $G^\top q_t\leq\alpha$ holds for every search episode $t\leq\tau$. The returned pair $(\widehat\pi^\diamond,\widehat\beta)$ satisfies $G^\top q^{P,\widehat\pi^\diamond}\leq\widehat\beta$ and has known margin $\min_{i\in[m]}(\alpha_i-\widehat\beta_i)\geq\rho/2$.
\end{theorem}
Theorem~\ref{thm:identification} controls exploration trajectories used for estimation and the baseline episodes required for safety.  Using only $1-\lambda_n\geq d/L$ gives $\tO(L^3|X|^2|A|/(d\rho^2))$ episodes. 
We obtain the sharper bound by controlling both counts through the margin gap.

To control the exploration count, define $\Delta_n:=\overline\rho_n-\underline\rho_n$ for each search policy whose cost upper bounds have been computed. On the confidence event, every policy selected for exploration has $\Delta_n>\rho/2$, by the stopping rule. The following cumulative gap bound therefore limits the number of exploration trajectories.
\begin{lemma}[Cumulative margin gap]\label{lem:main-counts}
Under the conditions of Theorem~\ref{thm:identification}, with probability at least $1-\delta/2$, %
for every integer $n\in[N]$, it holds that $\sum_{j=1}^{n}\Delta_j\leq\tO(L|X|\sqrt{|A|n})$.
\end{lemma}
Lemma~\ref{lem:main-counts} bounds the cumulative gap between the margin estimates.
To prove Lemma~\ref{lem:main-counts}, we compare the occupancies used for optimistic search and pessimistic cost evaluation with $q_n$. Both use $\widehat\pi_n$, so their deviations from $q_n$ are controlled by transition confidence widths weighted by $q_n$. Adding the cost widths bounds $\Delta_n$ by a conditional expected sum of widths along an exploration trajectory. We then use visit counts and martingale concentration to bound the cumulative gap.

We next bound the baseline executions needed to collect these trajectories. Since $\widehat\pi_n$ and $\lambda_n$ remain fixed until the next exploration trajectory, the number of preceding baseline episodes has conditional mean $\lambda_n/(1-\lambda_n)$. Writing $[z]_+=\max\{z,0\}$, the mixing rule gives:
\begin{equation}
\frac{\lambda_n}{1-\lambda_n}\leq\frac{[-\underline\rho_n]_+}{d}.
\label{eq:main-gap-mixing}
\end{equation}
Here $[-\underline\rho_n]_+$ is the largest constraint violation bound computed from $w_n$. If $\underline\rho_n\geq0$, the policy can be executed without the baseline even when the stopping condition is not yet met. Let $\tau_n$ denote the episode of the $n$-th exploration trajectory. The next lemma gives a bound on the baseline count $\tau_n-n$.
\begin{lemma}[Baseline executions]\label{lem:main-baseline-count}
Under the conditions of Theorem~\ref{thm:identification}, with probability at least $1-\delta/2$, the bound $\tau_n-n\leq(2/d)\sum_{j=1}^{n}[-\underline\rho_j]_+ +(4L/d)\log(4/\delta)$ holds for every integer $n\in [N]$.
\end{lemma}
Lemma~\ref{lem:main-baseline-count} limits the number of baseline episodes needed to collect these exploration trajectories. Its proof combines a supermartingale bound for the waiting times with Equation~\eqref{eq:main-gap-mixing}. %
We have $[-\underline\rho_n]_+\leq\Delta_n$ on the confidence event, so the margin gap also controls baseline use.

We prove Theorem~\ref{thm:identification} on the joint event of the two lemmas, of probability at least $1-\delta$. The stopping rule gives $\sum_{j=1}^n\Delta_j>n\rho/2$ for every $n\in [N]$. Lemma~\ref{lem:main-counts} then implies that $N$ is finite and $N\leq\tO(L^2|X|^2|A|/\rho^2)$. For $N\geq1$, Lemma~\ref{lem:main-counts} and the bound on $N$ give $\sum_{j=1}^N\Delta_j\leq\tO(L^2|X|^2|A|/\rho)$. Lemma~\ref{lem:main-baseline-count} and $d\leq\rho$ yield the bound on $\tau=\tau_N$. If $N=0$, then $\tau=0$. The stopping rules ensure that the returned baseline has a margin
of at least $\rho/2$. During the search, the algorithm uses the pessimistic
cost estimates to mix the exploration policy with the initial
baseline so that all constraints are satisfied in every episode.

To complete the proof of Theorem~\ref{thm:combined}, we combine the search regret bound $L\min\{T,\tau\}$ with Theorem~\ref{thm:sops}. If $\tau<T$, we condition on the history at search termination and apply Theorem~\ref{thm:sops} with the returned margin of at least $\rho/2$. Its guarantee holds for all feasible comparators with high probability, so both phases can use the same comparator in Equation~\eqref{eq:regret}. This gives the following bound:
\[
R_T \leq L\tau + \widetilde{\mathcal{O}}\!\left(
\frac{L^2|X|\sqrt{|A|(T-\tau)}}{\rho}
+
\frac{L|X|^3|A|}{\rho}
\right).
\]
If $\tau \geq T$, then $R_T \leq L\tau$.
Combining Theorem~\ref{thm:identification} with $d \leq L \leq |X|$
and allocating $\delta/2$ to each phase proves
Theorem~\ref{thm:combined}.

\section{Regret Lower Bound}\label{sec:lower}
In this section, we show that both margin-dependent terms in Theorem~\ref{thm:combined} are necessary for algorithms that satisfy the constraints in every episode. The full proofs are given in Appendix~\ref{app:lower}.
\begin{theorem}[Lower bound on the regret]\label{thm:lower}
Let $0<d\leq\rho\leq1/8$, $\delta\in(0,1/16]$, and let $T\geq1/(d\rho)$ be an integer. For any algorithm that, on every instance satisfying
Assumption~\ref{ass:baseline} with initial margin $d$ and
Slater margin $\rho$, ensures with probability at least $1-\delta$ that $G^\top q_t\leq\alpha$ holds for every $t\in[T]$, there exists an instance with $L=m=1$, $|A|=3$, and margins $d$ and $\rho$ on which, with probability at least $3/4-2\delta$, it holds:
\[
R_T=\Omega\!\left(\frac{\sqrt T}{\rho}+\frac1{d\rho}\right).
\]
\end{theorem}
Theorem~\ref{thm:lower} shows that the regret bound in Theorem~\ref{thm:combined} is tight in $T$, $d$, and $\rho$, up to logarithmic factors.
The additive $1/(d\rho)$ term captures an unavoidable regret cost of exploration from the initial baseline under hard constraints.
MA-OPS confines the dependence on $d$ to this additive cost while attaining the optimal $\sqrt T/\rho$ leading term without prior knowledge of $\rho$.

\begin{remark}
For adversarial CMDPs, \citet{stradi2025} prove an $\Omega(\sqrt T/\rho)$ regret lower bound for algorithms whose cumulative positive constraint violation is $o(\sqrt T)$.
In constrained MABs with adversarial losses, \citet{genalti2025datadependent} derive a loss-dependent lower bound on expected regret with worst-case term $\Omega(\sqrt T/\rho)$.
For algorithms guaranteeing zero violation, we establish $\Omega(\sqrt T/\rho+1/(d\rho))$ regret with constant probability (Theorem~\ref{thm:lower}), identifying an additional cost associated with the initial margin $d$.
In stochastic MABs, \citet{pacchiano2025} prove an expected-regret lower bound containing $\Omega((1-r_b)/d^2)$ for sufficiently large $T$, where $r_b\in(0,1)$ is the baseline mean reward.
Their construction for this term has $d=\rho$ in our notation, whereas our $\Omega(1/(d\rho))$ term characterizes the joint margin dependence when $d<\rho$.
\end{remark}

The proof of Theorem~\ref{thm:lower} uses two pairs of instances with the associated margins $d$ and $\rho$, each sharing the same baseline and cost bounds.
For the $\sqrt T/\rho$ term, a low-loss action has mean cost at the threshold in one instance and $\Theta(1/\sqrt T)$ above it in the other.
Safety in the latter requires probability $\Omega(1/(\rho\sqrt T))$ on higher-loss actions per episode.
The observation distributions remain close over $T$ episodes, so change of measure transfers this loss bound with constant probability to the instance where this action is feasible and optimal.
For the $1/(d\rho)$ term when $d\leq\rho/4$, we exchange the cost distributions of two low-loss actions so that a different action has margin $\rho$ in each instance.
A policy feasible in both assigns only $\mathcal O(d/\rho)$ total probability to these actions.
Each such observation contributes $\mathcal O(\rho^2)$ relative entropy, so information per episode in the common feasible set is $\mathcal O(d\rho)$.
We compare the distributions of histories ending with the selection of the first policy outside this set.
By change of measure and safety in both instances, the learner remains in this set for $\Omega(1/(d\rho))$ episodes with constant probability, each incurring constant regret from baseline use.
Choosing the first construction when $d\sqrt T\geq1/2$ and the second otherwise yields the claimed bound.

\bibliographystyle{plainnat}
\bibliography{example_paper2}

@book{Altman1999ConstrainedMD,
  author    = {Altman, Eitan},
  title     = {Constrained {Markov} Decision Processes},
  publisher = {Chapman and Hall/CRC},
  year      = {1999},
  doi       = {10.1201/9781315140223}
}

@inproceedings{stradi2025,
  author={Stradi, Francesco Emanuele and Castiglioni, Matteo and Marchesi, Alberto and Gatti, Nicola},
  title={Learning Adversarial {MDPs} with Stochastic Hard Constraints},
  booktitle={Proceedings of the 42nd International Conference on Machine Learning},
  series={Proceedings of Machine Learning Research},
  volume={267},
  pages={56920--56951},
  year={2025},
  url={https://proceedings.mlr.press/v267/stradi25a.html}
}

@inproceedings{jin2020,
  author={Jin, Chi and Jin, Tiancheng and Luo, Haipeng and Sra, Suvrit and Yu, Tiancheng},
  title={Learning Adversarial {Markov} Decision Processes with Bandit Feedback and Unknown Transition},
  booktitle={Proceedings of the 37th International Conference on Machine Learning},
  series={Proceedings of Machine Learning Research},
  volume={119},
  pages={4860--4869},
  year={2020},
  url={https://proceedings.mlr.press/v119/jin20c.html}
}

@article{bollini2026,
  author={Bollini, Matteo and Genalti, Gianmarco and Stradi, Francesco Emanuele and Castiglioni, Matteo and Marchesi, Alberto},
  title={Replicable Constrained Bandits},
  journal={arXiv preprint arXiv:2602.14580},
  year={2026},
  url={https://arxiv.org/abs/2602.14580}
}

@inproceedings{pacchiano2021,
  author={Pacchiano, Aldo and Ghavamzadeh, Mohammad and Bartlett, Peter and Jiang, Heinrich},
  title={Stochastic Bandits with Linear Constraints},
  booktitle={Proceedings of the 24th International Conference on Artificial Intelligence and Statistics},
  series={Proceedings of Machine Learning Research},
  volume={130},
  pages={2827--2835},
  year={2021},
  url={https://proceedings.mlr.press/v130/pacchiano21a.html}
}

@book{lattimore2020,
  author={Lattimore, Tor and Szepesv{\'a}ri, Csaba},
  title={Bandit Algorithms},
  publisher={Cambridge University Press},
  year={2020},
  url={https://tor-lattimore.com/downloads/book/book.pdf}
}

@inproceedings{maurer2009,
  author={Maurer, Andreas and Pontil, Massimiliano},
  title={Empirical {Bernstein} Bounds and Sample Variance Penalization},
  booktitle={Proceedings of the 22nd Annual Conference on Learning Theory},
  year={2009},
  url={https://arxiv.org/abs/0907.3740}
}

@inproceedings{liu2021zero,
  author = {Liu, Tao and Zhou, Ruida and Kalathil, Dileep and Kumar, P. R. and Tian, Chao},
  title = {Learning Policies with Zero or Bounded Constraint Violation for Constrained {MDPs}},
  booktitle = {Advances in Neural Information Processing Systems},
  volume = {34},
  pages = {17183--17193},
  year = {2021},
  url = {https://proceedings.neurips.cc/paper_files/paper/2021/hash/8ec2ba5e96ec1c050bc631abda80f269-Abstract.html}
}

@inproceedings{bura2022dope,
  author = {Bura, Archana and HasanzadeZonuzy, Aria and Kalathil, Dileep and Shakkottai, Srinivas and Chamberland, Jean-Francois},
  title = {{DOPE}: Doubly Optimistic and Pessimistic Exploration for Safe Reinforcement Learning},
  booktitle = {Advances in Neural Information Processing Systems},
  volume = {35},
  pages = {1047--1059},
  year = {2022},
  url = {https://proceedings.neurips.cc/paper_files/paper/2022/hash/076a93fd42aa85f5ccee921a01d77dd5-Abstract-Conference.html}
}

@article{yu2025improved,
  author = {Yu, Kihyun and Lee, Duksang and Overman, William and Lee, Dabeen},
  title = {Improved Regret Bound for Safe Reinforcement Learning via Tighter Cost Pessimism and Reward Optimism},
  journal = {Reinforcement Learning Journal},
  volume = {6},
  pages = {493--546},
  year = {2025},
  url = {https://rlj.cs.umass.edu/2025/papers/Paper41.html}
}

@inproceedings{kitamura2025episodewise,
  author = {Kitamura, Toshinori and Ghosh, Arnob and Kozuno, Tadashi and Kumagai, Wataru and Kasaura, Kazumi and Hoshino, Kenta and Hosoe, Yohei and Matsuo, Yutaka},
  title = {Provably Efficient {RL} under Episode-Wise Safety in Constrained {MDPs} with Linear Function Approximation},
  booktitle = {Advances in Neural Information Processing Systems},
  volume = {38},
  pages = {47141--47190},
  year = {2025},
  url = {https://proceedings.neurips.cc/paper_files/paper/2025/hash/435e8fbbfc2c6072d4f3a5cb6e56a39a-Abstract-Conference.html}
}

@inproceedings{genalti2025datadependent,
  author = {Genalti, Gianmarco and Stradi, Francesco Emanuele and Castiglioni, Matteo and Marchesi, Alberto and Gatti, Nicola},
  title = {Data-Dependent Regret Bounds for Constrained {MABs}},
  booktitle = {Advances in Neural Information Processing Systems},
  volume = {38},
  pages = {1568--1605},
  year = {2025},
  url = {https://papers.nips.cc/paper_files/paper/2025/hash/0266d95023740481d22d437aa8aba0e9-Abstract-Conference.html}
}

@inproceedings{muller2024truly,
  author = {M{\"u}ller, Adrian and Alatur, Pragnya and Cevher, Volkan and Ramponi, Giorgia and He, Niao},
  title = {Truly No-Regret Learning in Constrained {MDPs}},
  booktitle = {Proceedings of the 41st International Conference on Machine Learning},
  series = {Proceedings of Machine Learning Research},
  volume = {235},
  pages = {36605--36653},
  year = {2024},
  url = {https://proceedings.mlr.press/v235/muller24b.html}
}

@inproceedings{stradi2025strong,
  author = {Stradi, Francesco Emanuele and Castiglioni, Matteo and Marchesi, Alberto and Gatti, Nicola},
  title = {Optimal Strong Regret and Violation in Constrained {MDPs} via Policy Optimization},
  booktitle = {The Thirteenth International Conference on Learning Representations},
  pages = {52056--52081},
  year = {2025},
  url = {https://proceedings.iclr.cc/paper_files/paper/2025/hash/8105a80cea247fa6f8441a8069b65f2a-Abstract-Conference.html}
}

@inproceedings{stradi2025policy,
  author = {Stradi, Francesco Emanuele and Lunghi, Anna and Castiglioni, Matteo and Marchesi, Alberto and Gatti, Nicola},
  title = {Policy Optimization for {CMDPs} with Bandit Feedback: Learning Stochastic and Adversarial Constraints},
  booktitle = {Proceedings of the 42nd International Conference on Machine Learning},
  series = {Proceedings of Machine Learning Research},
  volume = {267},
  pages = {56952--56989},
  year = {2025},
  url = {https://proceedings.mlr.press/v267/stradi25b.html}
}

@inproceedings{yu2026primaldual,
  author = {Yu, Kihyun and Bae, Seoungbin and Lee, Dabeen},
  title = {Primal-Dual Policy Optimization for Linear {CMDPs} with Adversarial Losses},
  booktitle = {The Fourteenth International Conference on Learning Representations},
  pages = {49875--49923},
  year = {2026},
  url = {https://proceedings.iclr.cc/paper_files/paper/2026/hash/51f547584cd1fcb87114ea022822a60d-Abstract-Conference.html}
}

@inproceedings{zimin2013online,
  author = {Zimin, Alexander and Neu, Gergely},
  title = {Online Learning in Episodic {Markovian} Decision Processes by Relative Entropy Policy Search},
  booktitle = {Advances in Neural Information Processing Systems},
  volume = {26},
  year = {2013},
  url = {https://proceedings.neurips.cc/paper/2013/hash/68053af2923e00204c3ca7c6a3150cf7-Abstract.html}
}

@inproceedings{luo2021policy,
  author = {Luo, Haipeng and Wei, Chen-Yu and Lee, Chung-Wei},
  title = {Policy Optimization in Adversarial {MDPs}: Improved Exploration via Dilated Bonuses},
  booktitle = {Advances in Neural Information Processing Systems},
  volume = {34},
  pages = {22931--22942},
  year = {2021},
  url = {https://proceedings.neurips.cc/paper/2021/hash/c1b8bf9e071c0dabb899e7a27f353762-Abstract.html}
}

@inproceedings{jin2021best,
  author = {Jin, Tiancheng and Huang, Longbo and Luo, Haipeng},
  title = {The Best of Both Worlds: Stochastic and Adversarial Episodic {MDPs} with Unknown Transition},
  booktitle = {Advances in Neural Information Processing Systems},
  volume = {34},
  pages = {20491--20502},
  year = {2021},
  url = {https://proceedings.neurips.cc/paper/2021/hash/abb9d15b3293a96a3ea116867b2b16d5-Abstract.html}
}

@inproceedings{li2024improved,
  author = {Li, Long-Fei and Zhao, Peng and Zhou, Zhi-Hua},
  title = {Improved Algorithm for Adversarial Linear Mixture {MDPs} with Bandit Feedback and Unknown Transition},
  booktitle = {Proceedings of The 27th International Conference on Artificial Intelligence and Statistics},
  series = {Proceedings of Machine Learning Research},
  volume = {238},
  pages = {3061--3069},
  year = {2024},
  url = {https://proceedings.mlr.press/v238/li24n.html}
}

@inproceedings{ni2025safe,
  author = {Ni, Tingting and Kamgarpour, Maryam},
  title = {A Safe Exploration Approach to Constrained {Markov} Decision Processes},
  booktitle = {Proceedings of The 28th International Conference on Artificial Intelligence and Statistics},
  series = {Proceedings of Machine Learning Research},
  volume = {258},
  pages = {3592--3600},
  year = {2025},
  url = {https://proceedings.mlr.press/v258/ni25a.html}
}

@inproceedings{sridharan2025unknown,
  author = {Sridharan, Karthik and Yoo, Seung Won Wilson},
  title = {Online Learning with Unknown Constraints},
  booktitle = {Proceedings of the 42nd International Conference on Machine Learning},
  series = {Proceedings of Machine Learning Research},
  volume = {267},
  pages = {56790--56819},
  year = {2025},
  url = {https://proceedings.mlr.press/v267/sridharan25a.html}
}

@inproceedings{stradi2024online,
  author = {Stradi, Francesco Emanuele and Germano, Jacopo and Genalti, Gianmarco and Castiglioni, Matteo and Marchesi, Alberto and Gatti, Nicola},
  title = {Online Learning in {CMDP}s: Handling Stochastic and Adversarial Constraints},
  booktitle = {Proceedings of the 41st International Conference on Machine Learning},
  series = {Proceedings of Machine Learning Research},
  volume = {235},
  pages = {46692--46721},
  year = {2024},
  url = {https://proceedings.mlr.press/v235/stradi24a.html}
}

@inproceedings{stradi2025taming,
  author = {Stradi, Francesco Emanuele and Lunghi, Anna and Castiglioni, Matteo and Marchesi, Alberto and Gatti, Nicola},
  title = {Taming Adversarial Constraints in {CMDPs}},
  booktitle = {Advances in Neural Information Processing Systems},
  volume = {38},
  pages = {9697--9745},
  year = {2025},
  url = {https://proceedings.neurips.cc/paper_files/paper/2025/hash/0df684f8bb7e8c0bcdb8ea49be5035b3-Abstract-Conference.html}
}

@inproceedings{bernasconi2024beyond,
  author = {Bernasconi, Martino and Castiglioni, Matteo and Celli, Andrea and Fusco, Federico},
  title = {Beyond Primal-Dual Methods in Bandits with Stochastic and Adversarial Constraints},
  booktitle = {Advances in Neural Information Processing Systems},
  volume = {37},
  pages = {8541--8568},
  year = {2024},
  url = {https://proceedings.neurips.cc/paper_files/paper/2024/hash/0fd5675f49141c79ad22d7a533c89b12-Abstract-Conference.html}
}

@inproceedings{bernasconi2025noregret,
  author = {Bernasconi, Martino and Castiglioni, Matteo and Celli, Andrea},
  title = {No-Regret is not enough! {Bandits} with General Constraints through Adaptive Regret Minimization},
  booktitle = {Proceedings of the 42nd International Conference on Machine Learning},
  series = {Proceedings of Machine Learning Research},
  volume = {267},
  pages = {3877--3898},
  year = {2025},
  url = {https://proceedings.mlr.press/v267/bernasconi25a.html}
}

@article{pacchiano2025,
  author = {Pacchiano, Aldo and Ghavamzadeh, Mohammad and Bartlett, Peter},
  title = {Contextual Bandits with Stage-wise Constraints},
  journal = {Journal of Machine Learning Research},
  volume = {26},
  number = {170},
  pages = {1--57},
  year = {2025},
  url = {https://www.jmlr.org/papers/v26/24-0267.html}
}

@article{kaufmann2023,
  author = {Kaufmann, Elia and Bauersfeld, Leonard and Loquercio, Antonio and M{\"u}ller, Matthias and Koltun, Vladlen and Scaramuzza, Davide},
  title = {Champion-level drone racing using deep reinforcement learning},
  journal = {Nature},
  volume = {620},
  pages = {982--987},
  year = {2023},
  doi = {10.1038/s41586-023-06419-4},
  url = {https://www.nature.com/articles/s41586-023-06419-4}
}

@article{silver2016,
  author = {Silver, David and Huang, Aja and Maddison, Chris J. and Guez, Arthur and Sifre, Laurent and van den Driessche, George and Schrittwieser, Julian and Antonoglou, Ioannis and Panneershelvam, Veda and Lanctot, Marc and Dieleman, Sander and Grewe, Dominik and Nham, John and Kalchbrenner, Nal and Sutskever, Ilya and Lillicrap, Timothy and Leach, Madeleine and Kavukcuoglu, Koray and Graepel, Thore and Hassabis, Demis},
  title = {Mastering the game of {Go} with deep neural networks and tree search},
  journal = {Nature},
  volume = {529},
  pages = {484--489},
  year = {2016},
  doi = {10.1038/nature16961},
  url = {https://www.nature.com/articles/nature16961}
}

@inproceedings{ouyang2022,
  author = {Ouyang, Long and Wu, Jeffrey and Jiang, Xu and Almeida, Diogo and Wainwright, Carroll and Mishkin, Pamela and Zhang, Chong and Agarwal, Sandhini and Slama, Katarina and Ray, Alex and Schulman, John and Hilton, Jacob and Kelton, Fraser and Miller, Luke and Simens, Maddie and Askell, Amanda and Welinder, Peter and Christiano, Paul F. and Leike, Jan and Lowe, Ryan},
  title = {Training language models to follow instructions with human feedback},
  booktitle = {Advances in Neural Information Processing Systems},
  volume = {35},
  pages = {27730--27744},
  year = {2022},
  url = {https://proceedings.neurips.cc/paper_files/paper/2022/hash/b1efde53be364a73914f58805a001731-Abstract-Conference.html}
}

@book{sutton2018,
  author = {Sutton, Richard S. and Barto, Andrew G.},
  title = {Reinforcement Learning: An Introduction},
  edition = {Second},
  publisher = {MIT Press},
  year = {2018},
  url = {https://mitpress.mit.edu/9780262039246/reinforcement-learning/}
}

@inproceedings{qiu2020,
  author = {Qiu, Shuang and Wei, Xiaohan and Yang, Zhuoran and Ye, Jieping and Wang, Zhaoran},
  title = {Upper Confidence Primal-Dual Reinforcement Learning for {CMDP} with Adversarial Loss},
  booktitle = {Advances in Neural Information Processing Systems},
  volume = {33},
  pages = {15277--15287},
  year = {2020},
  url = {https://proceedings.neurips.cc/paper_files/paper/2020/hash/ae95296e27d7f695f891cd26b4f37078-Abstract.html}
}

@inproceedings{zhu2025anytime,
  author = {Zhu, Jiahui and Yu, Kihyun and Lee, Dabeen and Liu, Xin and Wei, Honghao},
  title = {An Optimistic Algorithm for online {CMDPs} with Anytime Adversarial Constraints},
  booktitle = {Proceedings of the 42nd International Conference on Machine Learning},
  series = {Proceedings of Machine Learning Research},
  volume = {267},
  pages = {80347--80372},
  year = {2025},
  url = {https://proceedings.mlr.press/v267/zhu25ab.html}
}

@inproceedings{vaswani2022,
  author = {Vaswani, Sharan and Yang, Lin and Szepesv{\'a}ri, Csaba},
  title = {Near-Optimal Sample Complexity Bounds for Constrained {MDPs}},
  booktitle = {Advances in Neural Information Processing Systems},
  volume = {35},
  pages = {3110--3122},
  year = {2022},
  url = {https://proceedings.neurips.cc/paper_files/paper/2022/hash/14a5ebc9cd2e507cd811df78c15bf5d7-Abstract-Conference.html}
}

@article{stradi2025beyondslater,
  author = {Stradi, Francesco Emanuele and Chiefari, Eleonora Fidelia and Castiglioni, Matteo and Marchesi, Alberto and Gatti, Nicola},
  title = {Beyond {Slater}'s Condition in Online {CMDPs} with Stochastic and Adversarial Constraints},
  journal = {arXiv preprint arXiv:2509.20114},
  year = {2025},
  url = {https://arxiv.org/abs/2509.20114}
}

@inproceedings{ding2021safe,
  author = {Ding, Dongsheng and Wei, Xiaohan and Yang, Zhuoran and Wang, Zhaoran and Jovanovic, Mihailo},
  title = {Provably Efficient Safe Exploration via Primal-Dual Policy Optimization},
  booktitle = {Proceedings of The 24th International Conference on Artificial Intelligence and Statistics},
  series = {Proceedings of Machine Learning Research},
  volume = {130},
  pages = {3304--3312},
  year = {2021},
  url = {https://proceedings.mlr.press/v130/ding21d.html}
}

@inproceedings{wei2022tripleq,
  author = {Wei, Honghao and Liu, Xin and Ying, Lei},
  title = {{Triple-Q}: A Model-Free Algorithm for Constrained Reinforcement Learning with Sublinear Regret and Zero Constraint Violation},
  booktitle = {Proceedings of The 25th International Conference on Artificial Intelligence and Statistics},
  series = {Proceedings of Machine Learning Research},
  volume = {151},
  pages = {3274--3307},
  year = {2022},
  url = {https://proceedings.mlr.press/v151/wei22a.html}
}

@inproceedings{ding2023nonstationary,
  author = {Ding, Yuhao and Lavaei, Javad},
  title = {Provably Efficient Primal-Dual Reinforcement Learning for {CMDPs} with Non-stationary Objectives and Constraints},
  booktitle = {Proceedings of the AAAI Conference on Artificial Intelligence},
  volume = {37},
  pages = {7396--7404},
  year = {2023},
  doi = {10.1609/aaai.v37i6.25900}
}

@inproceedings{wei2023nonstationary,
  author = {Wei, Honghao and Ghosh, Arnob and Shroff, Ness and Ying, Lei and Zhou, Xingyu},
  title = {Provably Efficient Model-Free Algorithms for Non-stationary {CMDPs}},
  booktitle = {Proceedings of The 26th International Conference on Artificial Intelligence and Statistics},
  series = {Proceedings of Machine Learning Research},
  volume = {206},
  pages = {6527--6570},
  year = {2023},
  url = {https://proceedings.mlr.press/v206/wei23b.html}
}

@inproceedings{li2023conservative,
  author = {Li, Donghao and Huang, Ruiquan and Shen, Cong and Yang, Jing},
  title = {Near-optimal Conservative Exploration in Reinforcement Learning under Episode-wise Constraints},
  booktitle = {Proceedings of the 40th International Conference on Machine Learning},
  series = {Proceedings of Machine Learning Research},
  volume = {202},
  pages = {19527--19564},
  year = {2023},
  url = {https://proceedings.mlr.press/v202/li23k.html}
}

@inproceedings{amani2019safe,
  author = {Amani, Sanae and Alizadeh, Mahnoosh and Thrampoulidis, Christos},
  title = {Linear Stochastic Bandits Under Safety Constraints},
  booktitle = {Advances in Neural Information Processing Systems},
  volume = {32},
  year = {2019}
}

@book{boyd2004convex,
  author={Boyd, Stephen and Vandenberghe, Lieven},
  title={Convex Optimization},
  publisher={Cambridge University Press},
  year={2004}
}

@article{badanidiyuru2018knapsacks,
  author={Badanidiyuru, Ashwinkumar and Kleinberg, Robert and Slivkins, Aleksandrs},
  title={Bandits with Knapsacks},
  journal={Journal of the ACM},
  volume={65},
  number={3},
  year={2018},
  url={https://arxiv.org/abs/1305.2545}
}

@article{efroni2020exploration,
  author={Efroni, Yonathan and Mannor, Shie and Pirotta, Matteo},
  title={Exploration-Exploitation in Constrained {MDPs}},
  journal={arXiv preprint arXiv:2003.02189},
  year={2020},
  url={https://arxiv.org/abs/2003.02189}
}

@inproceedings{gangrade2025feasible,
  author={Gangrade, Aditya and Pacchiano, Aldo and Scott, Clayton and Saligrama, Venkatesh},
  title={Feasible Action Search for Bandit Linear Programs via {Thompson} Sampling},
  booktitle={Proceedings of the 42nd International Conference on Machine Learning},
  series={Proceedings of Machine Learning Research},
  volume={267},
  pages={18228--18256},
  year={2025},
  url={https://proceedings.mlr.press/v267/gangrade25a.html}
}

@inproceedings{hutchinson2024directional,
  author={Hutchinson, Spencer and Turan, Berkay and Alizadeh, Mahnoosh},
  title={Directional Optimism for Safe Linear Bandits},
  booktitle={Proceedings of The 27th International Conference on Artificial Intelligence and Statistics},
  series={Proceedings of Machine Learning Research},
  volume={238},
  pages={658--666},
  year={2024},
  url={https://proceedings.mlr.press/v238/hutchinson24a.html}
}

@inproceedings{hutchinson2025optimistic,
  author={Hutchinson, Spencer and Chen, Tianyi and Alizadeh, Mahnoosh},
  title={Optimistic Safety for Online Convex Optimization with Unknown Linear Constraints},
  booktitle={Proceedings of The 28th International Conference on Artificial Intelligence and Statistics},
  series={Proceedings of Machine Learning Research},
  volume={258},
  pages={2809--2817},
  year={2025},
  url={https://proceedings.mlr.press/v258/hutchinson25a.html}
}

@inproceedings{liu2025onlinecomplexity,
  author={Liu, Chang and Li, Yunfan and Yang, Lin F.},
  title={Near-Optimal Sample Complexity for Online Constrained {MDPs}},
  booktitle={Advances in Neural Information Processing Systems},
  volume={38},
  year={2025},
  url={https://proceedings.neurips.cc/paper_files/paper/2025/hash/445500190fe42e9f58f65199d2bd70d1-Abstract-Conference.html}
}

@inproceedings{liu2026linearcomplexity,
  author={Liu, Xingtu and Yang, Lin F. and Vaswani, Sharan},
  title={Sample Complexity Bounds for Linear Constrained {MDPs} with a Generative Model},
  booktitle={Proceedings of The 37th International Conference on Algorithmic Learning Theory},
  series={Proceedings of Machine Learning Research},
  volume={313},
  pages={1--70},
  year={2026},
  url={https://proceedings.mlr.press/v313/liu26b.html}
}

@article{mannor2009samplepath,
  author = {Mannor, Shie and Tsitsiklis, John N. and Yu, Jia Yuan},
  title = {Online Learning with Sample Path Constraints},
  journal = {Journal of Machine Learning Research},
  volume = {10},
  number = {20},
  pages = {569--590},
  year = {2009},
  url = {https://jmlr.org/papers/v10/mannor09a.html}
}

@inproceedings{wei2026averagecomplexity,
  author={Wei, Yukuan and Li, Xudong and Yang, Lin F.},
  title={Near-Optimal Sample Complexity Bounds for Constrained Average-Reward {MDPs}},
  booktitle={International Conference on Learning Representations},
  year={2026},
  url={https://proceedings.iclr.cc/paper_files/paper/2026/file/ceadbf97decaba4054a48f45d9a61c62-Paper-Conference.pdf}
}

@inproceedings{wu2016conservative,
  author={Wu, Yifan and Shariff, Roshan and Lattimore, Tor and Szepesv{\'a}ri, Csaba},
  title={Conservative Bandits},
  booktitle={Proceedings of The 33rd International Conference on Machine Learning},
  series={Proceedings of Machine Learning Research},
  volume={48},
  pages={1254--1262},
  year={2016},
  url={https://proceedings.mlr.press/v48/wu16.html}
}

@inproceedings{zheng2020constrained,
  title = {Constrained Upper Confidence Reinforcement Learning},
  author = {Zheng, Liyuan and Ratliff, Lillian},
  booktitle = {Proceedings of the 2nd Conference on Learning for Dynamics and Control},
  pages = {620--629},
  volume = {120},
  series = {Proceedings of Machine Learning Research},
  year = {2020},
  publisher = {PMLR},
  url = {https://proceedings.mlr.press/v120/zheng20a.html}
}

\clearpage
\appendix
\numberwithin{equation}{section}
\begingroup
\hypersetup{linktoc=page}
\etocsettagdepth{main}{none}
\etocsettagdepth{appendix}{subsection}
\etocdepthtag{appendix}
\renewcommand{\contentsname}{Appendix Contents}
\setcounter{tocdepth}{2}
\pdfbookmark[0]{Appendix Contents}{appendix-contents}
\tableofcontents
\endgroup
\clearpage
\section{Additional Related Work}\label{app:related}

In this section, we provide further discussion of related work on online learning and reinforcement learning.

\paragraph{Online CMDPs with cumulative violation}
In stochastic CMDPs with unknown transitions, \citet{efroni2020exploration} use linear programming to bound strong regret and cumulative positive violation with $\widetilde{\mathcal O}(\sqrt T)$ dependence on $T$.
For primal-dual policy optimization, \citet{muller2024truly} establish sublinear bounds for both quantities, and \citet{stradi2025strong} attain the $\widetilde{\mathcal O}(\sqrt T)$ rate in $T$.
Under stochastic constraints, \citet{stradi2025beyondslater} remove Slater's condition for regret and cumulative positive violation.
For cumulative signed violation, guarantees range from sublinear bounds \citep{qiu2020} to nonpositive violation after sufficiently many episodes under Slater's condition \citep{wei2022tripleq}.
For adversarial linear CMDPs, \citet{yu2026primaldual} control cumulative violation with full-information losses and bandit cost observations.
With changing constraints, the hindsight optimum under average constraints can be unattainable with fully adversarial inputs \citep{mannor2009samplepath}.
This motivates competitive guarantees \citep{stradi2024online,stradi2025policy} or restrictions on temporal variation for dynamic regret \citep{ding2023nonstationary,wei2023nonstationary}.
With non-stationary rewards and costs, \citet{stradi2025taming} use a fixed comparator feasible on average, whereas \citet{zhu2025anytime} study stochastic rewards and adversarial constraints against a fixed policy feasible in every episode.

\paragraph{Online CMDPs with hard constraints}
\citet{zheng2020constrained} study feasibility during learning under known dynamics and average-cost constraints.
For stochastic episodic CMDPs with unknown transitions, \citet{liu2021zero}, \citet{bura2022dope}, and \citet{yu2025improved} use a known strictly feasible baseline.
Known-baseline guarantees also extend to linear function approximation \citep{kitamura2025episodewise} and discounted CMDPs \citep{ni2025safe}.
Conservative reinforcement learning constrains each episode's expected reward relative to a baseline \citep{li2023conservative}.
A known baseline provides more information than Slater's condition alone, which also serves to bound optimal dual multipliers in primal-dual analyses \citep{ding2021safe}.
The main known-baseline analyses use exact expected costs \citep{liu2021zero,bura2022dope,yu2025improved,stradi2025}, so input and actual baseline margins coincide.
Our formulation allows upper bounds on baseline costs.
A small $d$ may therefore reflect conservative cost information even when the actual baseline margin is large.

\paragraph{Constrained bandits}
Budget and long-term formulations couple resource use across rounds \citep{badanidiyuru2018knapsacks,bernasconi2024beyond,bernasconi2025noregret}, while conservative bandits impose a cumulative reward requirement relative to a baseline \citep{wu2016conservative}.
For per-round constraints, feasibility may apply to an individual action \citep{amani2019safe,hutchinson2024directional} or to the expected cost of a randomized decision \citep{pacchiano2025}.
The latter allows positive probability on individually infeasible actions and is the relevant formulation for our policy mixtures.
Several results for randomized decisions assume known baseline costs \citep{pacchiano2021,pacchiano2025}.
\citet{pacchiano2025} also assume a known baseline reward in their main upper-bound analysis, treating unknown baseline values separately.
For adversarial losses, \citet{genalti2025datadependent} derive data-dependent bounds using a baseline with known costs.
Their main analysis uses a maximum-margin baseline, while their results also extend to a general strictly feasible baseline.
Beyond bandits, related work studies per-round feasibility in online convex optimization with linear constraints \citep{hutchinson2025optimistic} and through regression and online learning oracles for more general constraint classes \citep{sridharan2025unknown}.

\paragraph{Sample complexity under constraints}
Sample-complexity guarantees have been established with a generative model in discounted tabular CMDPs \citep{vaswani2022}, discounted linear CMDPs \citep{liu2026linearcomplexity}, and average-reward CMDPs \citep{wei2026averagecomplexity}, as well as through online interaction \citep{liu2025onlinecomplexity}.
For exact feasibility of the output, these bounds exhibit inverse-square dependence on the Slater margin without requiring every data-collection policy to be feasible.
For unknown linear constraints, \citet{gangrade2025feasible} find a nearly maximum-margin action or detect infeasibility while controlling cumulative positive violation during the search.
We bound the regret of learning a larger-margin baseline while keeping every episode feasible.

\paragraph{Adversarial MDPs}
Occupancy-measure methods provide a basis for learning in episodic MDPs \citep{zimin2013online}.
For unknown transitions and bandit feedback, \citet{jin2020} combine online mirror descent, transition confidence sets, and upper occupancy bounds.
Subsequent developments include policy optimization with dilated bonuses \citep{luo2021policy}, guarantees for stochastic and adversarial losses \citep{jin2021best}, and linear mixture models \citep{li2024improved}.
These works study MDPs without episode-wise cost constraints.

\section{Occupancy Measures and Confidence Bounds}\label{app:confidence}
We establish the confidence and occupancy error bounds used in the analyses of S-OPS and baseline search.
\subsection{Occupancy Measures}\label{app:occupancy}
We use the occupancy notation from Section~\ref{sec:setting}. Inner products, coordinatewise comparisons, and $\ell_1$ norms use state--action marginals, while flow constraints and relative entropy use the full transition occupancies. At a state with $q(x)=0$, define $\pi^q(a\mid x)=1/|A|$.
Following \citet[Appendix A.1]{jin2020}, we represent an occupancy by its nonnegative entries $q(x,a,x')$ on consecutive layers, with marginals $q(x,a)=\sum_{x'}q(x,a,x')$. These entries satisfy
\begin{equation}
\begin{gathered}
\begin{aligned}
\sum_{a\in A,\,x'\in X_1}q(x_0,a,x')&=1,\\
\sum_{a\in A,\,x'\in X_{k+1}}q(x,a,x')
&=\sum_{y\in X_{k-1},\,a\in A}q(y,a,x),
\quad x\in X_k,\quad 1\leq k<L,
\end{aligned}\\
\bigl(\widehat P(x'\mid x,a)-\epsilon(x,a,x')\bigr)q(x,a)
\leq q(x,a,x')
\leq\bigl(\widehat P(x'\mid x,a)+\epsilon(x,a,x')\bigr)q(x,a),
\end{gathered}
\label{eq:occupancy-flow}
\end{equation}
where the last line holds for every valid transition and uses the estimates and confidence bounds from Appendix~\ref{app:estimation}. These linear constraints define the convex polytope $\Delta(\cP)$. Adding $q(x,a)=\pi(a\mid x)\sum_bq(x,b)$ gives $\Delta(\cP,\pi)$. Consequently, both the optimistic margin problem and each fixed-policy cost maximization in \eqref{eq:id-cost} are linear programs. Each fixed-policy problem maximizes the total expected cost over a single kernel $P'\in\cP$. For the upper occupancy bounds in S-OPS, the maximizing kernel may depend on $(x,a)$.
\begin{lemma}\label{lem:realization}
Fix a Markov policy $\pi$, and let $\cP$ be a transition confidence set defined in Appendix~\ref{app:estimation}. For every $q\in\Delta(\cP,\pi)$, there exists $P^q\in\cP$ such that $q=q^{P^q,\pi}$.
\end{lemma}
\begin{proof}
For $q(x,a)>0$, set $P^q(x'\mid x,a)=q(x,a,x')/q(x,a)$. The flow and confidence constraints ensure that $P^q(\cdot\mid x,a)$ is a probability distribution satisfying the confidence bounds at $(x,a)$. If $q(x,a)=0$, choose any transition distribution satisfying these bounds. Such a distribution exists by Appendix~\ref{app:estimation}. Since the confidence constraints apply separately to each state--action pair, the constructed kernel belongs to $\cP$. Nonnegativity gives $q(x,a)P^q(x'\mid x,a)=q(x,a,x')$ in both cases.
The initial state has probability $q(x_0)=1$. Suppose each state $x\in X_k$ is reached with probability $q(x)$, where $k<L-1$. The fixed-policy constraints and flow conservation then give
\begin{align*}
\Prob_{P^q,\pi}(x_{k+1}=x')
&=\sum_{x\in X_k,a}q(x)\pi(a\mid x)P^q(x'\mid x,a)\\
&=\sum_{x\in X_k,a}q(x,a,x')=q(x').
\end{align*}
By induction, $q(x)$ equals the probability of visiting each nonterminal state $x$. Thus $q^{P^q,\pi}(x,a,x')=q(x)\pi(a\mid x)P^q(x'\mid x,a)=q(x,a,x')$.
\end{proof}
In particular, each $\widehat q\in\Delta(\cP)$ is induced by $\pi^{\widehat q}$ and some $P'\in\cP$.
\paragraph{Optimistic margin search}
We introduce a scalar $z\in\mathbb R$ and write \eqref{eq:optimistic-margin-search} as the linear program
\begin{equation}
\begin{aligned}
\max_{q,z}\quad & z\\
\text{subject to}\quad
&q\in\Delta(\cP_{n-1}),\\
&(\widehat g_{n-1,i}-\xi_{n-1})^\top q+z\leq\alpha_i,
\qquad i\in[m],\\
&z\leq L.
\end{aligned}
\label{eq:optimistic-margin-lp}
\end{equation}
For fixed $q$, maximizing over $z$ recovers the objective in \eqref{eq:optimistic-margin-search}. This objective is continuous on the nonempty compact set $\Delta(\cP_{n-1})$, so both problems attain the same maximum.
\subsection{Estimators and Confidence Bounds}\label{app:estimation}
Let $\cF_t$ denote the $\sigma$-algebra generated by the interaction history through episode $t$ and the loss vectors $\ell_1,\ldots,\ell_t$. The Markov policies and mixing probabilities used in episode $t$ are $\cF_{t-1}$-measurable. We write $\E_t[Z]=\E[Z\mid\cF_{t-1},\ell_t]$. Conditional on $\cF_{t-1}$ and $\ell_t$, the cost matrix $G_t$ has distribution $\mathcal G$. It is independent of the policy draw and the random variables used to sample actions and state transitions.\par
Index the episodes used for estimation by $j$. Let $I_j(x,a)$ and $J_j(x,a,x')$ indicate whether episode $j$ in this sequence visits $(x,a)$ and takes the transition $(x,a,x')$, respectively. With $N(x,a)=\sum_j I_j(x,a)$ and $M(x,a,x')=\sum_jJ_j(x,a,x')$, define
\begin{equation}\begin{aligned}
\widehat g_i(x,a)&=\frac{\sum_jG_{j,i}(x,a)I_j(x,a)}{\max\{1,N(x,a)\}},\\
\widehat P(x'\mid x,a)&=\frac{M(x,a,x')}{\max\{1,N(x,a)\}}.
\label{eq:estimates}
\end{aligned}\end{equation}
Here $G_{j,i}(x,a)$ is the cost of constraint $i$ at $(x,a)$ in episode $j$ of the estimation sequence. For S-OPS, $j$ counts all episodes in the current phase, so $G_{j,i}=g_{j,i}$. During baseline search, $j$ counts exploration episodes, so $G_{j,i}=g_{\tau_j,i}$, where $\tau_j$ is their global episode index. For a logarithmic parameter $u$, set
\begin{align}
\xi(x,a)&=\min\left\{1,\sqrt{\frac{4u}{\max\{1,N(x,a)\}}}\right\},\label{eq:cost-radius}\\*
\epsilon(x,a,x')&=\min\left\{1,
2\sqrt{\frac{\widehat P(x'\mid x,a)u}{\max\{1,N(x,a)-1\}}}
+\frac{14u}{3\max\{1,N(x,a)-1\}}\right\}.
\label{eq:transition-radius}
\end{align}
The confidence set $\cP$ consists of all transition kernels $P'$ satisfying $|P'(x'\mid x,a)-\widehat P(x'\mid x,a)|\leq\epsilon(x,a,x')$ for $x\in X_k$, $a\in A$, $x'\in X_{k+1}$, and $0\leq k<L$. At every visited state--action pair, the empirical transition distribution satisfies these bounds. At an unvisited pair, every transition distribution satisfies them. Thus $\cP$ is nonempty. With the logarithmic parameters used below, initialization gives $\widehat G_0=0$ and $\xi_0(x,a)=1$ for every $(x,a)$. The set $\cP_0$ then contains every transition kernel allowed by the layered model.
For S-OPS, define the confidence event
\begin{equation}
\cE_T=\{P\in\cP_t,\ |\widehat G_t-G|\leq\Xi_t
\text{ for every }0\leq t\leq T\}.
\label{eq:sops-confidence-event}
\end{equation}
For baseline search, we allow the phase to continue beyond $T$ until a stopping condition is met. Let $\cE_{\mathrm{s}}$ be the event that $P\in\cP_{n-1}$ and $|\widehat G_{n-1}-G|\leq\Xi_{n-1}$ at initialization and after every update. The next lemma bounds the failure probability in each phase, with $\delta$ denoting its confidence parameter.
\begin{lemma}\label{lem:confidence}
Under the observation model in Section~\ref{sec:setting}, S-OPS with $u=\log\frac{8m|X|^2|A|T}{\delta}$ satisfies $\Prob(\cE_T)\geq1-\delta/8$, and baseline search with $u=\log\frac{m|X|^2|A|(n+1)^3}{\delta}$ satisfies $\Prob(\cE_{\mathrm{s}})\geq1-\delta/8$.
\end{lemma}
\begin{proof}
Fix $(x,a)\in X_k\times A$. For $1\leq r\leq N(x,a)$, let $j_r=\min\{j:\sum_{s=1}^j I_s(x,a)=r\}$ be the index of its $r$-th visit. By the observation model and the use of a policy fixed throughout each episode, visiting $(x,a)$ is independent of the current cost matrix conditional on past observations. For each fixed $i$, we can therefore view the costs $G_{j_r,i}(x,a)$ as successive samples from an i.i.d. sequence with mean $g_i(x,a)$. For each fixed $x'\in X_{k+1}$, the indicators $J_{j_r}(x,a,x')$ are similarly obtained from successive Bernoulli samples with mean $P(x'\mid x,a)$. We apply concentration at each fixed visit count $r$ and then take a union bound over the counts.
\paragraph{Cost estimates}
For each $i\in[m]$, let $\widehat g_i(x,a)$ denote the average of the first $r\geq1$ cost samples. Hoeffding's inequality gives, for any $u>0$,
\begin{equation*}
\Prob\!\left(|\widehat g_i(x,a)-g_i(x,a)|>\sqrt{4u/r}\right)\leq2e^{-8u}\leq2e^{-u}.
\end{equation*}
With no observations, $\widehat g_i(x,a)=0$ and $\xi(x,a)=1$. Since both the estimate and the mean cost belong to $[0,1]$, truncating the confidence bound at one preserves the inequality.
\paragraph{Transition probabilities}
Fix $x'\in X_{k+1}$ and write $p=P(x'\mid x,a)$. For $r\geq2$, let $\widehat p$ be the empirical mean of the first $r$ next-state indicators. Their sample variance is $r\widehat p(1-\widehat p)/(r-1)$. We apply the empirical Bernstein inequality of \citet[Theorem~4]{maurer2009} to the indicator and its complement, each with failure probability $e^{-u}$. With probability at least $1-2e^{-u}$,
\begin{align*}
|\widehat p-p|
&\leq\sqrt{\frac{2\widehat p(1-\widehat p)(u+\log2)}{r-1}}
+\frac{7(u+\log2)}{3(r-1)}\\
&\leq2\sqrt{\frac{\widehat p u}{r-1}}+\frac{14u}{3(r-1)}.
\end{align*}
Both choices of $u$ satisfy $u\geq1$. The second inequality then follows from $1-\widehat p\leq1$ and $u+\log2\leq2u$. Since $|\widehat p-p|\leq1$, truncating the bound at one gives \eqref{eq:transition-radius}. For $r=0,1$, $\epsilon(x,a,x')=1$.
\paragraph{Union bound}
For each coordinate and fixed positive visit count, the confidence bound fails with probability at most $2e^{-u}$. The layer sizes satisfy
\begin{align*}
\sum_{k=0}^{L-1}|X_k||X_{k+1}|
&\leq\left(\sum_{j=0}^{\lfloor L/2\rfloor}|X_{2j}|\right)
\left(\sum_{j=0}^{\lfloor(L-1)/2\rfloor}|X_{2j+1}|\right)
\leq\frac{|X|^2}{4}.
\end{align*}
Together with $|X|-1\leq|X|^2/4$ and $m\geq1$, this bounds the total number of cost and transition entries by
\[
m|A|(|X|-1)+|A|\sum_{k=0}^{L-1}|X_k||X_{k+1}|
\leq\frac{m|X|^2|A|}{2}.
\]
For S-OPS, take $u=\log\frac{8m|X|^2|A|T}{\delta}$ and sum over at most $T$ positive visit counts. The resulting failure probability satisfies
\[
\Prob(\cE_T^c)\leq m|X|^2|A|T\cdot\frac{\delta}{8m|X|^2|A|T}=\frac\delta8.
\]
For the search estimates, use $u=\log(m|X|^2|A|(r+2)^3/\delta)$ at visit count $r\geq1$. If $N_{n-1}(x,a)=r$, then $n\geq r+1$. The logarithmic parameter used at search index $n$ is therefore at least this large, so the same confidence bounds apply. The total failure probability over all visit counts is at most
\begin{align*}
m|X|^2|A|\sum_{r=1}^\infty\frac{\delta}{m|X|^2|A|(r+2)^3}
&=\delta\sum_{k=3}^\infty\frac1{k^3}\\
&<\delta\int_2^\infty x^{-3}\,dx=\frac\delta8.
\end{align*}
Thus $\Prob(\cE_{\mathrm{s}})\geq1-\delta/8$.
\end{proof}
\subsection{Occupancy Error Bounds}\label{app:occupancy-errors}
For a fixed policy, we bound the difference between its true occupancy measure and those induced by kernels in the confidence set. We also bound the error in the upper occupancy bounds used by S-OPS.
For a transition confidence set $\cP$ with confidence bounds $\epsilon$ and $x\in X_k$, define
\begin{equation}
b(x,a)=\min\left\{2,2\sum_{x'\in X_{k+1}}\epsilon(x,a,x')\right\}.
\label{eq:row-width-definition}
\end{equation}
For a Markov policy $\pi$, write $q=q^{P,\pi}$ and $u^\pi(x,a)=\max_{P'\in\cP}q^{P',\pi}(x,a)$.

\begin{lemma}\label{lem:fixed-policy-deviation}
If $P\in\cP$, then $\norm{q^{P',\pi}-q}_1\leq Lb^\top q$
for all $P'\in\cP$, and $\norm{u^\pi-q}_1\leq |X|b^\top q$.
\end{lemma}
\begin{proof}
Fix $P'\in\cP$ and write $q'=q^{P',\pi}$. Since both transition kernels belong to $\cP$, the confidence constraints give $\norm{P'(\cdot\mid x,a)-P(\cdot\mid x,a)}_1\leq b(x,a)$. For $0\leq k<L-1$ and $y\in X_{k+1}$, the flow constraints give
\begin{align*}
q'(y)-q(y)
={}&\sum_{x\in X_k,a}\bigl(q'(x,a)-q(x,a)\bigr)P'(y\mid x,a)+\sum_{x\in X_k,a}q(x,a)\bigl(P'(y\mid x,a)-P(y\mid x,a)\bigr).
\end{align*}
Both occupancies use $\pi$, so $q'(x,a)-q(x,a)=\pi(a\mid x)(q'(x)-q(x))$. Taking absolute values and summing over $y\in X_{k+1}$ gives
\[
\sum_{y\in X_{k+1}}|q'(y)-q(y)|
\leq\sum_{x\in X_k}|q'(x)-q(x)|
+\sum_{x\in X_k,a}q(x,a)b(x,a).
\]
The initial distributions agree. Iterating this inequality and summing over $k$ therefore yields
\begin{align*}
\norm{q'-q}_1
&=\sum_{k=0}^{L-1}\sum_{x\in X_k}|q'(x)-q(x)|\\
&\leq\sum_{h=0}^{L-2}(L-1-h)\sum_{x\in X_h,a}q(x,a)b(x,a)
\leq Lb^\top q.
\end{align*}
For the upper occupancy bound, fix $y\in X_k$. The same recurrence applies to every $P'\in\cP$, and $P\in\cP$, so
\[
0\leq\max_{P'\in\cP}q^{P',\pi}(y)-q(y)
\leq\sum_{h=0}^{k-1}\sum_{x\in X_h,a}q(x,a)b(x,a).
\]
Since $q^{P',\pi}(y,a)=\pi(a\mid y)q^{P',\pi}(y)$, summing first over actions and then over states gives
\begin{align*}
\norm{u^\pi-q}_1
&=\sum_{k=0}^{L-1}\sum_{y\in X_k}
\left(\max_{P'\in\cP}q^{P',\pi}(y)-q(y)\right)\\
&\leq\sum_{k=0}^{L-1}|X_k|\sum_{h=0}^{k-1}\sum_{x\in X_h,a}q(x,a)b(x,a)
\leq |X|b^\top q.
\end{align*}
\end{proof}
We next bound $b(x,a)$ in terms of the visit count $N(x,a)$. If $r=N(x,a)\geq2$, $x\in X_k$, and the logarithmic parameter is $u\geq1$, then Cauchy--Schwarz and $\sum_{x'}\widehat P(x'\mid x,a)=1$ give
\begin{equation}
\begin{aligned}
b(x,a)
&\leq\min\left\{2,
4\sqrt{\frac{|X_{k+1}|u}{r-1}}
+\frac{28|X_{k+1}|u}{3(r-1)}\right\}\\
&\leq\mathcal O\!\left(\min\left\{1,
\sqrt{\frac{|X_{k+1}|u}{\max\{1,r\}}}\right\}\right).
\label{eq:row-radius-rate}
\end{aligned}
\end{equation}
For the last inequality, use $r-1\geq r/2$ and separate the cases $|X_{k+1}|u/r\leq1$ and $|X_{k+1}|u/r>1$. For $r=0,1$, the last bound holds because $b(x,a)\leq2$ and $|X_{k+1}|u\geq1$.
\paragraph{Cumulative visit counts}
For a finite sequence of episodes, write $N_t(x,a)=\sum_{s=1}^t I_s(x,a)$. For each $(x,a)$,
\begin{equation}
\sum_{t=1}^T\frac{I_t(x,a)}{\sqrt{\max\{1,N_{t-1}(x,a)\}}}
=\sum_{r=0}^{N_T(x,a)-1}\frac1{\sqrt{\max\{1,r\}}}
\leq3\sqrt{N_T(x,a)},
\label{eq:visit-count-sum}
\end{equation}
where an empty sum is zero. The equality follows by indexing the visits to $(x,a)$, and the inequality follows by integration. During baseline search, we apply this bound to the $n$ exploration episodes.

\section{Analysis of S-OPS}\label{app:sops}
We first establish zero constraint violation on $\cE_T$. We then combine the estimation bounds with the mirror descent analysis to prove the regret guarantee in Theorem~\ref{thm:sops}.

\subsection{Algorithm Details}\label{app:sops-specification}
Algorithm~\ref{alg:sops} uses baseline cost upper bounds $\beta$ satisfying Assumption~\ref{ass:baseline} and the parameters in Theorem~\ref{thm:sops}. The estimates and confidence sets are defined in Appendix~\ref{app:estimation}.

\begin{algorithm}[!t]
\caption{S-OPS}\label{alg:sops}
\small
\begin{algorithmic}[1]
\Require $X,A,\alpha,T,\delta,\eta,\gamma,\pi^\diamond,\beta$
\For{$k\in\{0,\ldots,L-1\}$, $(x,a,x')\in X_k\times A\times X_{k+1}$}
  \State $N_0(x,a)\gets0$, $M_0(x,a,x')\gets0$
  \State $\widehat q_1(x,a,x')\gets\dfrac{1}{|X_k||A||X_{k+1}|}$
\EndFor
\State $\pi_1\gets\begin{cases}
  \pi^\diamond, & \text{w. probability }\lambda_0:=\displaystyle\max_{i\in[m]}\left\{\frac{L-\alpha_i}{L-\beta_i}\right\},\\[2pt]
  \pi^{\widehat q_1}, & \text{w. probability }1-\lambda_0
\end{cases}$
\For{$t\in[T]$}
  \State Execute $\pi_t$ for one episode (Section~\ref{sec:setting}) and receive feedback
  \State Build \emph{upper occupancy bounds} for $k\in\{0,\ldots,L-1\}$:
  \Statex \hfill$\displaystyle u_t(x_k,a_k)\gets\max_{P'\in\cP_{t-1}}q^{P',\pi_t}(x_k,a_k)$\hfill\mbox{}
  \State Build \emph{optimistic loss estimator} for $(x,a)\in X\times A$:
  \Statex \hfill$\displaystyle\widehat\ell_t(x,a)\gets\begin{cases}
    \dfrac{\ell_t(x,a)}{u_t(x,a)+\gamma}, & \text{if }I_t(x,a)=1,\\[2pt]
    0, & \text{otherwise}
  \end{cases}$\hfill\mbox{}
  \For{$k\in\{0,\ldots,L-1\}$}
    \State $N_t(x_k,a_k)\gets N_{t-1}(x_k,a_k)+1$
    \State $M_t(x_k,a_k,x_{k+1})\gets M_{t-1}(x_k,a_k,x_{k+1})+1$
  \EndFor
  \State Build $\cP_t$, $\widehat G_t$, and $\Xi_t$ using Appendix~\ref{app:estimation}
  \State Build \emph{unconstrained occupancy} for all $(x,a,x')$:
  \Statex \hfill$\displaystyle\widetilde q_{t+1}(x,a,x')\gets\widehat q_t(x,a,x')e^{-\eta\widehat\ell_t(x,a)}$\hfill\mbox{}
  \If{$\operatorname{PROJ}(\widetilde q_{t+1},\widehat G_t,\Xi_t,\cP_t)$ \emph{is feasible}}
    \State $\widehat q_{t+1}\gets\operatorname{PROJ}(\widetilde q_{t+1},\widehat G_t,\Xi_t,\cP_t)$
    \State $\widehat\pi_{t+1}\gets\pi^{\widehat q_{t+1}}$
    \State Build $\widehat u_{t+1}\in[0,1]^{|X\times A|}$ so that for all $(x,a)$:
    \Statex \hfill$\displaystyle\widehat u_{t+1}(x,a)\gets\max_{P'\in\cP_t}q^{P',\widehat\pi_{t+1}}(x,a)$\hfill\mbox{}
    \State Define $\overline m:=\{i\in[m]:(\widehat g_{t,i}+\xi_t)^\top\widehat u_{t+1}>\alpha_i\}$
    \State $\sigma\gets\displaystyle\max_{i\in\overline m}\left\{\frac{\min\{(\widehat g_{t,i}+\xi_t)^\top\widehat u_{t+1},L\}-\alpha_i}{\min\{(\widehat g_{t,i}+\xi_t)^\top\widehat u_{t+1},L\}-\beta_i}\right\}$
    \State $\lambda_t\gets\begin{cases}
      \sigma, & \text{if }\exists i\in[m]:(\widehat g_{t,i}+\xi_t)^\top\widehat u_{t+1}>\alpha_i,\\
      0, & \text{if }\forall i\in[m]:(\widehat g_{t,i}+\xi_t)^\top\widehat u_{t+1}\leq\alpha_i
    \end{cases}$
  \Else
    \State $\widehat q_{t+1}\gets\text{take any }q\in\Delta(\cP_t)$, $\quad\lambda_t\gets1$
  \EndIf
  \State $\pi_{t+1}\gets\begin{cases}
    \pi^\diamond, & \text{with probability }\lambda_t,\\
    \pi^{\widehat q_{t+1}}, & \text{with probability }1-\lambda_t
  \end{cases}$
\EndFor
\end{algorithmic}
\end{algorithm}

\paragraph{Loss estimation}
Write $\widehat\pi_t=\pi^{\widehat q_t}$. In episode $t$, the randomized policy $\pi_t$ selects $\pi^\diamond$ with probability $\lambda_{t-1}$ and $\widehat\pi_t$ otherwise, then follows the selected policy for all $L$ steps. For every transition kernel $P'$, its occupancy satisfies
\begin{equation}
q^{P',\pi_t}=\lambda_{t-1}q^{P',\pi^\diamond}
+(1-\lambda_{t-1})q^{P',\widehat\pi_t}.
\label{eq:mixture}
\end{equation}
Let $I_t(x,a)$ indicate whether episode $t$ visits $(x,a)$. We estimate the loss using an upper occupancy bound for this mixture:
\begin{equation}\begin{aligned}
u_t(x,a)&=\max_{P'\in\cP_{t-1}}q^{P',\pi_t}(x,a),\\
\widehat\ell_t(x,a)&=\frac{\ell_t(x,a)I_t(x,a)}{u_t(x,a)+\gamma}.
\label{eq:sops-loss-estimator}
\end{aligned}\end{equation}
Counts at unvisited state--action pairs and transitions retain their previous values. We update $\widehat G_t,\Xi_t,\cP_t$ using observations from all episodes $1,\ldots,t$, including baseline episodes, with the estimates and confidence bounds in \eqref{eq:estimates}--\eqref{eq:transition-radius} and $u=\log\frac{8m|X|^2|A|T}{\delta}$.

\paragraph{Occupancy update}
We use $\widehat\ell_t$ in the multiplicative update
\begin{equation}
\widetilde q_{t+1}(x,a,x')=\widehat q_t(x,a,x')e^{-\eta\widehat\ell_t(x,a)},
\label{eq:sops-update}
\end{equation}
for every valid transition. The projection uses the unnormalized relative entropy
\begin{equation}\begin{aligned}
D(q\|q')&=\sum_{x,a,x'}q(x,a,x')\log\frac{q(x,a,x')}{q'(x,a,x')}-\sum_{x,a,x'}(q(x,a,x')-q'(x,a,x')).
\label{eq:entropy}
\end{aligned}\end{equation}
We use $0\log(0/z)=0$ for $z\geq0$ and $z\log(z/0)=+\infty$ for $z>0$.
Write $\Xi_t=(\xi_t,\ldots,\xi_t)$. The updated cost and transition estimates define the optimistic projection set
\begin{equation}
\cC_t=\{q\in\Delta(\cP_t):(\widehat G_t-\Xi_t)^\top q\leq\alpha\}.
\label{eq:projection-set}
\end{equation}
We write $\operatorname{PROJ}(\widetilde q_{t+1},\widehat G_t,\Xi_t,\cP_t)$ for a minimizer of $D(q\|\widetilde q_{t+1})$ over $\cC_t$. When the projection problem attains a finite minimum, Algorithm~\ref{alg:sops} uses a minimizing occupancy measure. Otherwise, it sets $\lambda_t=1$.

\paragraph{Cost bounds}
When the projection problem attains a finite minimum, let
\[
\widehat u_{t+1}(x,a)=\max_{P'\in\cP_t}q^{P',\widehat\pi_{t+1}}(x,a).
\]
We then choose the baseline probability for the next episode using the pessimistic cost bounds:
\begin{align}
w_{t+1,i}&=\min\{(\widehat g_{t,i}+\xi_t)^\top\widehat u_{t+1},L\},\label{eq:sops-cost}\\*
\lambda_t&=\begin{cases}
0, & w_{t+1}\leq\alpha,\\[2pt]
\displaystyle\max_{i:\,w_{t+1,i}>\alpha_i}
\frac{w_{t+1,i}-\alpha_i}{w_{t+1,i}-\beta_i}, & \text{otherwise}.
\end{cases}
\label{eq:sops-mixing}
\end{align}
We set the auxiliary value $\sigma$ in Algorithm~\ref{alg:sops} to zero when $\overline m=\varnothing$, consistently with the zero branch for $\lambda_t$. The algorithm initializes this rule with the cost upper bound $L$. For the analysis, also define $\widehat u_1(x,a)=\max_{P'\in\cP_0}q^{P',\widehat\pi_1}(x,a)$.

\subsection{Zero Constraint Violation}

\begin{lemma}\label{lem:sops-safety}
On $\cE_T$, every projection in S-OPS attains a finite minimum. Moreover, for every $t\in[T]$, we have $G^\top q_t\leq\alpha$, $1-\lambda_{t-1}\geq d/L$, and $\widehat q_t\leq(L/d)u_t$.
\end{lemma}
\begin{proof}
Fix any $q\in\Delta(P)$ with $G^\top q\leq\alpha$. On $\cE_T$, we have $q\in\Delta(\cP_t)$ and
\[
(\widehat G_t-\Xi_t)^\top q\leq G^\top q\leq\alpha.
\]
Thus $q\in\cC_t$ at every episode.
\paragraph{Projection step}
At initialization, $D(q\|\widehat q_1)<\infty$ because $\widehat q_1$ is positive on every valid transition. Suppose $D(q\|\widehat q_t)<\infty$. The multiplicative update has the same support, so $D(q\|\widetilde q_{t+1})<\infty$. Since $\cC_t$ is compact and the objective is lower semicontinuous and finite at $q$, the projection attains a finite minimum. If a minimizing occupancy $r$ were zero at a coordinate where $q$ is positive, the feasible perturbation $r_s=(1-s)r+sq$ would satisfy
\[
D(r_s\|\widetilde q_{t+1})-D(r\|\widetilde q_{t+1})
=s\log s\sum_{j:\,r_j=0}q_j+\mathcal O(s)<0
\]
for sufficiently small $s>0$, where $j$ indexes valid transitions. The negative $s\log s$ term contradicts minimality. Hence $D(q\|\widehat q_{t+1})<\infty$. Induction establishes this property at every iteration for every feasible comparator.
\paragraph{Exploration probability}
For every constraint with $w_{t,i}>\alpha_i$, the function $s\mapsto(s-\alpha_i)/(s-\beta_i)$ is increasing on $(\alpha_i,L]$. Hence $\lambda_{t-1}\leq\max_i(L-\alpha_i)/(L-\beta_i)=\lambda_0$. Therefore
\[
1-\lambda_{t-1}\geq1-\lambda_0
=\min_i\frac{\alpha_i-\beta_i}{L-\beta_i}\geq\frac dL.
\]
Choose $P_t\in\cP_{t-1}$ such that $\widehat q_t=q^{P_t,\widehat\pi_t}$. Equation~\eqref{eq:mixture} gives
\begin{equation}
\widehat q_t=q^{P_t,\widehat\pi_t}
\leq\frac{q^{P_t,\pi_t}}{1-\lambda_{t-1}}
\leq\frac Ld q^{P_t,\pi_t}\leq\frac Ld u_t.
\label{eq:sops-domination}
\end{equation}
\paragraph{Zero constraint violation}
In the first episode, $g_i^\top q_1\leq\lambda_0\beta_i+(1-\lambda_0)L\leq\alpha_i$. At every later episode, the confidence bounds imply $g_i^\top q^{P,\widehat\pi_t}\leq w_{t,i}$. If $w_{t,i}\leq\alpha_i$, then $g_i^\top q_t\leq\lambda_{t-1}\beta_i+(1-\lambda_{t-1})w_{t,i}\leq\alpha_i$. Otherwise $\lambda_{t-1}\geq(w_{t,i}-\alpha_i)/(w_{t,i}-\beta_i)$, so
\begin{align*}
g_i^\top q_t
&\leq\lambda_{t-1}\beta_i+(1-\lambda_{t-1})w_{t,i}\\
&=w_{t,i}-\lambda_{t-1}(w_{t,i}-\beta_i)\leq\alpha_i.
\end{align*}
\end{proof}
\subsection{Occupancy-Weighted Estimation Errors}\label{app:sops-errors}\label{app:weighted-occupancy}
We prove Lemma~\ref{lem:main-weighted-errors} by combining bounds on weighted occupancy errors and cumulative cost confidence widths.

Suppose the learner selects between Markov policies $\pi_{t,0},\pi_{t,1}$ with probabilities $p_{t,0},p_{t,1}$, all $\cF_{t-1}$-measurable. The resulting occupancy measure is $q_t=\sum_{j=0}^1p_{t,j}q^{P,\pi_{t,j}}$. Construct the S-OPS confidence sets using observations from all episodes and set $u_{t,j}(x,a)=\max_{P'\in\cP_{t-1}}q^{P',\pi_{t,j}}(x,a)$.
\begin{lemma}[Weighted occupancy error bounds]\label{lem:weighted-occupancy}
With probability at least $1-3\delta/8$, $\cE_T$ holds and
\[
\begin{aligned}
&\sum_{t=1}^T\sum_{j=0}^1p_{t,j}
\left(\norm{u_{t,j}-q^{P,\pi_{t,j}}}_1+
\norm{q^{P_{t,j},\pi_{t,j}}-q^{P,\pi_{t,j}}}_1\right)\\
&\qquad\leq\mathcal O\!\left(
L|X|\sqrt{|A|T\log\frac{8m|X|^2|A|T}{\delta}}
+|X|^3|A|\log^2\frac{8m|X|^2|A|T}{\delta}\right)
\end{aligned}
\]
for all $P_{t,j}\in\cP_{t-1}$. On the same event, the left-hand side is also at most
$\mathcal O\!\left(|X|^2\sqrt{|A|T\log\frac{8m|X|^2|A|T}{\delta}}\right)$.
\end{lemma}
\begin{proof}
We use the occupancy expansion of \citet[Appendix B.2]{jin2020} and retain both execution weights and squared-width terms.
\paragraph{Transition errors}
On $\cE_T$, if $\xi_{t-1}(x,a)<1$, then $N_{t-1}(x,a)>4\log\frac{8m|X|^2|A|T}{\delta}$. Equations~\eqref{eq:cost-radius}--\eqref{eq:transition-radius} and $N_{t-1}(x,a)-1\geq N_{t-1}(x,a)/2$ give
\[
\widehat P_{t-1}(x'\mid x,a)
\leq P(x'\mid x,a)
+\sqrt{2\widehat P_{t-1}(x'\mid x,a)}\,\xi_{t-1}(x,a)
+\frac73\xi_{t-1}(x,a)^2.
\]
Completing the square yields $\sqrt{\widehat P_{t-1}(x'\mid x,a)}\leq\sqrt{P(x'\mid x,a)}+3\xi_{t-1}(x,a)$. Both $P'$ and $P$ satisfy the confidence bounds. Substituting the preceding estimate into \eqref{eq:transition-radius} and applying the triangle inequality gives, for every $P'\in\cP_{t-1}$,
\[
|P'(x'\mid x,a)-P(x'\mid x,a)|
\leq16\left(\sqrt{P(x'\mid x,a)}\,\xi_{t-1}(x,a)+\xi_{t-1}(x,a)^2\right).
\]
The same bound holds when $\xi_{t-1}(x,a)=1$, since transition probabilities lie in $[0,1]$. Let $b_{t-1}$ be defined by \eqref{eq:row-width-definition} for $\cP_{t-1}$. For $x\in X_h$, \eqref{eq:row-radius-rate} also gives
\begin{equation}
\norm{P'(\cdot\mid x,a)-P(\cdot\mid x,a)}_1
\leq b_{t-1}(x,a)
\leq\mathcal O\!\left(\sqrt{|X_{h+1}|}\,\xi_{t-1}(x,a)\right).
\label{eq:transition-l1-width}
\end{equation}
Fix $t,j$, and write $\pi=\pi_{t,j}$, $q=q^{P,\pi}$, and $\xi=\xi_{t-1}$ in this part of the proof. Let $q^{P',\pi}(y\mid x')$ denote the probability of reaching $y$
from $x'$ under $(P',\pi)$, where $y$ is in the same or a later layer. 
Define $q^{P,\pi}(z,b\mid x')$ analogously for a state--action pair. 
For $y\in X_k$ with $1\leq k<L$ and $P'\in\cP_{t-1}$, telescoping the transition products gives
\[
q^{P',\pi}(y)-q(y)
=\sum_{h=0}^{k-1}\sum_{\substack{x\in X_h,a\\x'\in X_{h+1}}}
q(x,a)\bigl(P'(x'\mid x,a)-P(x'\mid x,a)\bigr)q^{P',\pi}(y\mid x').
\]
Take absolute values and apply the pointwise transition-error bound. Since $q^{P',\pi}(y\mid x')\leq1$, the terms involving $\xi(x,a)^2$ contribute at most
\[
16\sum_{h=0}^{k-1}|X_{h+1}|\sum_{x\in X_h,a}q(x,a)\xi(x,a)^2.
\]
In the remaining terms, use
\[
q^{P',\pi}(y\mid x')\leq q^{P,\pi}(y\mid x')+
\left|q^{P',\pi}(y\mid x')-q^{P,\pi}(y\mid x')\right|.
\]
For $x'\in X_{h+1}$, a second expansion, now starting at $x'$, gives
\begin{align*}
\left|q^{P',\pi}(y\mid x')-q^{P,\pi}(y\mid x')\right|
&\leq\sum_{r=h+1}^{k-1}\sum_{z\in X_r,b}
q^{P,\pi}(z,b\mid x')\norm{P'(\cdot\mid z,b)-P(\cdot\mid z,b)}_1\\
&\leq\mathcal O\!\left(\sum_{r=h+1}^{k-1}\sqrt{|X_{r+1}|}
\sum_{z\in X_r,b}q^{P,\pi}(z,b\mid x')\xi(z,b)\right).
\end{align*}
For $0\leq h<r<k$, Cauchy--Schwarz bounds the sum of products of confidence widths by
\begin{align*}
&\sum_{\substack{x\in X_h,a,\ x'\in X_{h+1}\\z\in X_r,b}}
q(x,a)\sqrt{P(x'\mid x,a)}\,\xi(x,a)q^{P,\pi}(z,b\mid x')\xi(z,b)\\
&\qquad\leq\sqrt{|X_{h+1}|
\left(\sum_{x\in X_h,a}q(x,a)\xi(x,a)^2\right)
\left(\sum_{z\in X_r,b}q(z,b)\xi(z,b)^2\right)}.
\end{align*}
We use weights $q(x,a)q^{P,\pi}(z,b\mid x')$ and the identities
\[
\sum_{z\in X_r,b}q^{P,\pi}(z,b\mid x')=1,
\qquad
\sum_{\substack{x\in X_h,a\\x'\in X_{h+1}}}
q(x,a)P(x'\mid x,a)q^{P,\pi}(z,b\mid x')=q(z,b).
\]
These estimates hold uniformly over $P'\in\cP_{t-1}$, so the maximizing kernel may be chosen separately for each target state $y$.
For the terms containing $q^{P,\pi}(y\mid x')$, we first sum over $y\in X_k$, using $\sum_{y\in X_k}q^{P,\pi}(y\mid x')=1$. Summing next over $k$ and applying $\sum_{x'\in X_{h+1}}\sqrt{P(x'\mid x,a)}\leq\sqrt{|X_{h+1}|}$ gives
\[
\mathcal O\!\left(L\sum_{h=0}^{L-1}\sqrt{|X_{h+1}|}\sum_{x\in X_h,a}q(x,a)\xi(x,a)\right).
\]
For the squared-width terms, summing over the target states gives
\begin{align*}
16\sum_{k=1}^{L-1}|X_k|\sum_{h=0}^{k-1}|X_{h+1}|\sum_{x\in X_h,a}q(x,a)\xi(x,a)^2
&\leq16|X|\sum_{h=0}^{L-1}|X_{h+1}|\sum_{x\in X_h,a}q(x,a)\xi(x,a)^2\\
&\leq16|X|^2\sum_{x,a}q(x,a)\xi(x,a)^2.
\end{align*}
Each pair of layers in the second expansion contributes to at most $|X|$ target states. The total contribution of these pairs is bounded by a universal constant times
\begin{align*}
&|X|\sum_{0\leq h<r<L}
\sqrt{|X_{h+1}|\sum_{x\in X_h,a}q(x,a)\xi(x,a)^2}
\sqrt{|X_{r+1}|\sum_{z\in X_r,b}q(z,b)\xi(z,b)^2}\\
&\qquad\leq\frac{|X|}{2}
\left(\sum_{h=0}^{L-1}\sqrt{|X_{h+1}|\sum_{x\in X_h,a}q(x,a)\xi(x,a)^2}\right)^2
\leq\frac{|X|^2}{2}\sum_{x,a}q(x,a)\xi(x,a)^2.
\end{align*}
The last inequality follows from Cauchy--Schwarz and $\sum_{h=0}^{L-1}|X_{h+1}|\leq|X|$. The initial-state error is zero. Since $q^{P',\pi}(y,a)=\pi(a\mid y)q^{P',\pi}(y)$, summing over actions gives the state marginal. This proves the required bound on $\norm{u_{t,j}-q}_1$. Lemma~\ref{lem:fixed-policy-deviation} and \eqref{eq:transition-l1-width} give
\begin{equation}
\norm{q^{P_{t,j},\pi}-q}_1\leq Lb_{t-1}^\top q
\leq\mathcal O\!\left(L\sum_{h=0}^{L-1}\sqrt{|X_{h+1}|}\sum_{x\in X_h,a}q(x,a)\xi(x,a)\right).
\label{eq:single-kernel-width}
\end{equation}
\paragraph{Weighted occupancy errors}
We weight both occupancy error bounds by $p_{t,j}$ and sum over $j$. Using $\sum_jp_{t,j}q^{P,\pi_{t,j}}=q_t$, we obtain
\begin{equation}
\begin{aligned}
&\sum_{j=0}^1p_{t,j}
\left(\norm{u_{t,j}-q^{P,\pi_{t,j}}}_1+
\norm{q^{P_{t,j},\pi_{t,j}}-q^{P,\pi_{t,j}}}_1\right)\\
&\qquad\leq\mathcal O\!\left(
L\sum_{k=0}^{L-1}\sqrt{|X_{k+1}|}\sum_{x\in X_k,a}q_t(x,a)\xi_{t-1}(x,a)
+|X|^2\sum_{x,a}q_t(x,a)\xi_{t-1}(x,a)^2\right).
\end{aligned}
\label{eq:weighted-error-estimates}
\end{equation}
Equations~\eqref{eq:cost-radius} and \eqref{eq:visit-count-sum} imply, on every observation sequence,
\begin{align*}
\sum_{t=1}^T\sum_{k=0}^{L-1}\sqrt{|X_{k+1}|}\sum_{x\in X_k,a}I_t(x,a)\xi_{t-1}(x,a)
&\leq6\sqrt{\log\frac{8m|X|^2|A|T}{\delta}}\sum_{k=0}^{L-1}\sqrt{|X_{k+1}|}\sum_{x\in X_k,a}\sqrt{N_T(x,a)}\\
&\leq6|X|\sqrt{|A|T\log\frac{8m|X|^2|A|T}{\delta}}.
\end{align*}
Here $\sum_{x\in X_k,a}N_T(x,a)=T$ and $\sum_k\sqrt{|X_k||X_{k+1}|}\leq|X|$.
Since $\E[I_t\mid\cF_{t-1}]=q_t$, replacing $I_t$ by $q_t$ adds a sum of martingale differences. Each has absolute value at most $\sum_k\sqrt{|X_{k+1}|}\leq|X|$. Azuma--Hoeffding thus gives, with probability at least $1-\delta/8$,
\begin{equation}
\sum_{t=1}^T\sum_{k=0}^{L-1}\sqrt{|X_{k+1}|}\sum_{x\in X_k,a}q_t(x,a)\xi_{t-1}(x,a)
\leq\mathcal O\!\left(|X|\sqrt{|A|T\log\frac{8m|X|^2|A|T}{\delta}}\right).
\label{eq:cumulative-weighted-width}
\end{equation}
On the intersection of this concentration event and $\cE_T$, combining \eqref{eq:single-kernel-width} and \eqref{eq:cumulative-weighted-width} gives
\begin{equation}
\sum_{t=1}^T\sum_{j=0}^1p_{t,j}
\norm{q^{P_{t,j},\pi_{t,j}}-q^{P,\pi_{t,j}}}_1
\leq\mathcal O\!\left(L|X|\sqrt{|A|T\log\frac{8m|X|^2|A|T}{\delta}}\right).
\label{eq:cumulative-single-kernel}
\end{equation}
For each $k$, the sum $\sum_{x\in X_k,a}I_t(x,a)\xi_{t-1}(x,a)^2$ lies in $[0,1]$, since exactly one pair in $X_k\times A$ is visited per episode. Using $e^{-z}\leq1-(1-e^{-1})z$ for $z\in[0,1]$ gives
\[
\E\!\left[\left.
\exp\!\left\{(1-e^{-1})\sum_{x\in X_k,a}q_t(x,a)\xi_{t-1}(x,a)^2
-\sum_{x\in X_k,a}I_t(x,a)\xi_{t-1}(x,a)^2\right\}
\,\right|\cF_{t-1}\right]\leq1.
\]
We iterate this conditional exponential bound over $t$ and apply Markov's inequality with failure probability $\delta/(8L)$ for each layer. A union bound gives
\[
\sum_{t=1}^T\sum_{x,a}q_t(x,a)\xi_{t-1}(x,a)^2
\leq2\sum_{t=1}^T\sum_{x,a}I_t(x,a)\xi_{t-1}(x,a)^2
+2L\log\frac{8L}{\delta}
\]
with probability at least $1-\delta/8$, where $(1-e^{-1})^{-1}<2$.
Indexing visits to a fixed $(x,a)$ in \eqref{eq:cost-radius} gives
\[
\sum_{t=1}^T I_t(x,a)\xi_{t-1}(x,a)^2
\leq4\log\frac{8m|X|^2|A|T}{\delta}\sum_{r=0}^{N_T(x,a)-1}\frac1{\max\{1,r\}}
\leq4\log\frac{8m|X|^2|A|T}{\delta}\left(2+\log T\right).
\]
Summing over $(x,a)$ and using $L<|X|$ and $\log T$, $\log(8L/\delta)\leq\log\frac{8m|X|^2|A|T}{\delta}$ gives, on the same event,
\begin{equation}
\sum_{t=1}^T\sum_{x,a}q_t(x,a)\xi_{t-1}(x,a)^2
\leq\mathcal O\!\left(|X||A|\log^2\frac{8m|X|^2|A|T}{\delta}\right).
\label{eq:cumulative-squared-width}
\end{equation}
\paragraph{Combining the bounds}
Summing \eqref{eq:weighted-error-estimates} over $t$ and applying \eqref{eq:cumulative-weighted-width} and \eqref{eq:cumulative-squared-width} proves the first bound. For the second bound, Lemma~\ref{lem:fixed-policy-deviation} bounds the same sum by $(|X|+L)\sum_t b_{t-1}^\top q_t$. We then apply \eqref{eq:transition-l1-width} and \eqref{eq:cumulative-weighted-width}. By Lemma~\ref{lem:confidence} and a union bound, $\cE_T$ and both concentration bounds hold together with probability at least $1-3\delta/8$.
\end{proof}

We next bound the cumulative cost confidence widths.
\begin{lemma}[Cumulative confidence widths]\label{lem:sops-width}
With probability at least $1-\delta/8$,
\[
\sum_{t=1}^T\xi_{t-1}^\top q_t
\leq\mathcal O\!\left(\sqrt{L|X||A|T\log\frac{8m|X|^2|A|T}{\delta}}\right).
\]
\end{lemma}
\begin{proof}
Because $\E_t[I_t]=q_t$, the differences $\xi_{t-1}^\top(q_t-I_t)$ have conditional mean zero and absolute value at most $L$. Azuma--Hoeffding gives
\[
\sum_t\xi_{t-1}^\top q_t
\leq\sum_t\xi_{t-1}^\top I_t+L\sqrt{2T\log(8/\delta)}.
\]
Equations~\eqref{eq:cost-radius} and \eqref{eq:visit-count-sum} give $\sum_{t=1}^T\xi_{t-1}(x,a)I_t(x,a)\leq6\sqrt{N_T(x,a)\log\frac{8m|X|^2|A|T}{\delta}}$ for each $(x,a)$. We sum over $(x,a)$ and apply Cauchy--Schwarz with $\sum_{x,a}N_T(x,a)=LT$ to obtain
\begin{align*}
\sum_{t=1}^T\xi_{t-1}^\top q_t
&\leq6\sqrt{L|X||A|T\log\frac{8m|X|^2|A|T}{\delta}}+L\sqrt{2T\log(8/\delta)}\\
&\leq\mathcal O\!\left(\sqrt{L|X||A|T\log\frac{8m|X|^2|A|T}{\delta}}\right),
\end{align*}
where the last line uses $L<|X|$, $|A|\geq1$, and $\log(8/\delta)\leq\log\frac{8m|X|^2|A|T}{\delta}$.
\end{proof}

\needspace{7\baselineskip}
\begin{proof}[Proof of Lemma~\ref{lem:main-weighted-errors}]
Intersect the events in Lemmas~\ref{lem:weighted-occupancy} and~\ref{lem:sops-width}. Their failure probabilities sum to $\delta/2$, and $\cE_T$ holds on the intersection. Apply Lemma~\ref{lem:weighted-occupancy} with policies $\pi^\diamond,\widehat\pi_t$ and probabilities $\lambda_{t-1},1-\lambda_{t-1}$. By Lemma~\ref{lem:realization}, we can choose $P_t\in\cP_{t-1}$ with $\widehat q_t=q^{P_t,\widehat\pi_t}$.
\paragraph{Constraint cost bounds}
For $t\geq2$, the projection constraint and $\norm{\widehat g_{t-1,i}-\xi_{t-1}}_\infty\leq1$ give, for every $i\in[m]$,
\begin{align*}
w_{t,i}-\alpha_i
&\leq(\widehat g_{t-1,i}-\xi_{t-1})^\top(\widehat u_t-\widehat q_t)
+2\xi_{t-1}^\top\widehat u_t\\
&\leq\norm{\widehat u_t-q^{P,\widehat\pi_t}}_1
+\norm{q^{P,\widehat\pi_t}-\widehat q_t}_1
+2\xi_{t-1}^\top\widehat u_t.
\end{align*}
The right-hand side is nonnegative and independent of $i$, so it also bounds $e_t$. Since $0\leq\xi_{t-1}\leq1$, \eqref{eq:mixture} implies
\[
(1-\lambda_{t-1})\xi_{t-1}^\top\widehat u_t
\leq\xi_{t-1}^\top q_t
+(1-\lambda_{t-1})\norm{\widehat u_t-q^{P,\widehat\pi_t}}_1.
\]
Multiply the bound on $e_t$ by $1-\lambda_{t-1}$ and sum over $t=2,\ldots,T$. Lemma~\ref{lem:weighted-occupancy} bounds both weighted occupancy errors, and Lemma~\ref{lem:sops-width} bounds the cumulative confidence widths.
\paragraph{Loss-estimator occupancy bounds}
Set $u_t^{\pi^\diamond}(x,a)=\max_{P'\in\cP_{t-1}}q^{P',\pi^\diamond}(x,a)$. On $\cE_T$, $q_t\leq u_t$. Bounding the maximum of the mixture by the sum of its componentwise maxima gives
\[
\norm{u_t-q_t}_1
\leq\lambda_{t-1}\norm{u_t^{\pi^\diamond}-q^{P,\pi^\diamond}}_1
+(1-\lambda_{t-1})\norm{\widehat u_t-q^{P,\widehat\pi_t}}_1.
\]
Summing over episodes and applying Lemma~\ref{lem:weighted-occupancy} bounds the cumulative occupancy error. Combining the two estimates gives
\[
\sum_{t=2}^T(1-\lambda_{t-1})e_t
+\sum_{t=1}^T\norm{u_t-q_t}_1
\leq\mathcal O\!\left(L|X|\sqrt{|A|T\log\frac{8m|X|^2|A|T}{\delta}}+|X|^3|A|\log^2\frac{8m|X|^2|A|T}{\delta}\right),
\]
on the same event of probability at least $1-\delta/2$.
The second conclusion of Lemma~\ref{lem:weighted-occupancy} also bounds this sum by $\mathcal O(|X|^2\sqrt{|A|T\log\frac{8m|X|^2|A|T}{\delta}})$ on the same event.
\end{proof}

\subsection{Proof of Theorem~\ref{thm:sops}}\label{app:sops-regret}
We first establish two concentration bounds for the loss estimates using the implicit-exploration argument of \citet{jin2020}.
\begin{lemma}[Loss estimation bounds]\label{lem:loss-estimates}
With probability at least $1-3\delta/8$, $\cE_T$ holds and
\begin{align*}
\sum_{t,x,a}\widehat\ell_t(x,a)
&\leq |X||A|T+\frac L\gamma\log\frac{8L}{\delta},\\
\sum_{t=1}^T(\widehat\ell_t-\ell_t)^\top q
&\leq\frac L\gamma\log\frac{8|X||A|}{\delta}
\end{align*}
for every $q\ge 0$ with $\sum_{x,a}q(x,a)=L$.
\end{lemma}
\begin{proof}
Set $s_t(x,a)=\max\{u_t(x,a),q_t(x,a)\}$ and $\overline\ell_t(x,a)=\ell_t(x,a)I_t(x,a)/(s_t(x,a)+\gamma)$. Then $s_t\geq q_t$ coordinatewise and $\overline\ell_t=\widehat\ell_t$ on $\cE_T$.
For $(x,a)$, write $\ell=\ell_t(x,a)$, $p=q_t(x,a)$, and $s=s_t(x,a)$. Set $z=\gamma\ell/(s+\gamma)\in[0,1]$. Since $e^z\leq1+z+z^2$ and $p\leq s$,
\begin{equation}
p(e^z-1)\leq p(z+z^2)
\leq\gamma\ell\frac{s(s+2\gamma)}{(s+\gamma)^2}
\leq\gamma\ell.
\label{eq:exp-estimator}
\end{equation}
The visit indicator is Bernoulli with conditional probability $p$, so
\[
\E_t[e^{\gamma(\overline\ell_t(x,a)-\ell_t(x,a))}]
\leq e^{-\gamma\ell_t(x,a)}(1+\gamma\ell_t(x,a))\leq1.
\]
By the tower property, the same exponential bound holds conditional on $\cF_{t-1}$. Iterating and applying Markov's inequality gives, with failure probability at most $\delta/(8|X||A|)$ for each coordinate,
\[
\sum_{t=1}^T(\overline\ell_t(x,a)-\ell_t(x,a))
\leq\frac1\gamma\log\frac{8|X||A|}{\delta}.
\]
A union bound over coordinates, followed by weighting each coordinate by $q(x,a)$ and summing, proves the second bound  for every nonnegative $q$ with total mass $L$.
For the first bound, use $\sum_{x\in X_k,a}I_t(x,a)=1$ for each $k=0,\ldots,L-1$. Summing \eqref{eq:exp-estimator} over the possible outcomes at step $k$ gives
\[
\E_t\exp\left\{\gamma\sum_{x\in X_k,a}
(\overline\ell_t(x,a)-\ell_t(x,a))\right\}\leq1.
\]
We iterate this exponential bound using the tower property and apply Markov's inequality with failure probability $\delta/(8L)$ for each $k$. Taking a union bound over the layers and summing over $k$ gives
\[
\sum_{t,x,a}(\overline\ell_t(x,a)-\ell_t(x,a))
\leq\frac L\gamma\log\frac{8L}{\delta}
\]
with failure probability at most $\delta/8$. Adding $\sum_{t,x,a}\ell_t(x,a)\leq|X||A|T$ proves the first bound for $\overline\ell_t$. By a union bound, both estimates and $\cE_T$ hold jointly with probability at least $1-3\delta/8$. On this event, $\overline\ell_t=\widehat\ell_t$, which proves the lemma.
\end{proof}
\begin{proof}[Proof of Theorem~\ref{thm:sops}]
The constraint guarantee follows from Lemma~\ref{lem:sops-safety}. We prove the regret bound on the intersection of $\cE_T$, the events in Lemmas~\ref{lem:weighted-occupancy},~\ref{lem:loss-estimates}, and~\ref{lem:sops-width}, and the martingale event used to bound $R^{(4)}$. Fix any feasible occupancy measure $q$. Select $P_t\in\cP_{t-1}$ so that $\widehat q_t=q^{P_t,\widehat\pi_t}$. Adding and subtracting the estimated losses and occupancies gives
\begin{equation}
\begin{aligned}
\sum_{t=1}^T\ell_t^\top(q_t-q)
={}&\underbrace{\sum_t\ell_t^\top(q_t-q^{P_t,\pi_t})}_{R^{(1)}}
+\underbrace{\sum_t\widehat\ell_t^\top(\widehat q_t-q)}_{R^{(2)}} +\underbrace{\sum_t\ell_t^\top(q^{P_t,\pi_t}-\widehat q_t)}_{R^{(3)}}\\&+\underbrace{\sum_t(\ell_t-\widehat\ell_t)^\top\widehat q_t}_{R^{(4)}}+\underbrace{\sum_t(\widehat\ell_t-\ell_t)^\top q}_{R^{(5)}}.
\label{eq:five-terms}
\end{aligned}
\end{equation}
The five terms correspond to transition estimation error, online mirror descent, regret from baseline use, and loss-estimation errors for the policy and comparator.
\paragraph{Transition error}
Equation~\eqref{eq:mixture} and $\norm{\ell_t}_\infty\leq1$ imply
\begin{equation}
\begin{aligned}R^{(1)}&\leq\sum_t\norm{q_t-q^{P_t,\pi_t}}_1\\
&\leq\sum_t\lambda_{t-1}\norm{q^{P,\pi^\diamond}-q^{P_t,\pi^\diamond}}_1
+\sum_t(1-\lambda_{t-1})\norm{q^{P,\widehat\pi_t}-\widehat q_t}_1\\
&\leq\mathcal O\!\left(L|X|\sqrt{|A|T\log\frac{8m|X|^2|A|T}{\delta}}\right),
\end{aligned}
\label{eq:transition-regret}
\end{equation}
where the last inequality follows from \eqref{eq:cumulative-single-kernel}. On the same event, Lemma~\ref{lem:main-weighted-errors} also gives
\begin{equation}
\sum_t\norm{u_t-q_t}_1\leq\mathcal O\!\left(L|X|\sqrt{|A|T\log\frac{8m|X|^2|A|T}{\delta}}+|X|^3|A|\log^2\frac{8m|X|^2|A|T}{\delta}\right).
\label{eq:mixture-upper-deviation}
\end{equation}
\paragraph{Online mirror descent}
Let $r=\widehat q_{t+1}$ and $y=\widetilde q_{t+1}$. By Lemma~\ref{lem:sops-safety}, both $D(q\|r)$ and $D(q\|y)$ are finite. Optimality of $r$ along the feasible segment $(1-s)r+sq$ gives
\begin{align*}
D(q\|y)-D(q\|r)-D(r\|y)
&=\sum_j(q_j-r_j)\log\frac{r_j}{y_j}\geq0.
\end{align*}
Terms with $q_j=r_j=0$ contribute zero. Combining this projection inequality with $e^{-x}\leq1-x+x^2/2$ for $x\geq0$ yields
\begin{align*}
\eta\widehat\ell_t^\top(\widehat q_t-q)
&=D(q\|\widehat q_t)-D(q\|\widetilde q_{t+1})
+D(\widehat q_t\|\widetilde q_{t+1})\\
&\leq D(q\|\widehat q_t)-D(q\|\widehat q_{t+1})
+\frac{\eta^2}{2}\sum_{x,a}\widehat q_t(x,a)\widehat\ell_t(x,a)^2.
\end{align*}
We sum over $t$ and telescope the divergence terms. The uniform initialization then gives
\begin{equation}
R^{(2)}\leq\frac{L\log(|X|^2|A|)}{\eta}
+\frac\eta2\sum_{t,x,a}\widehat q_t(x,a)\widehat\ell_t(x,a)^2.
\label{eq:omd-bound}
\end{equation}
The initial divergence bound follows from $q(x,a,x')\leq1$ and $1/\widehat q_1(x,a,x')\leq|X|^2|A|$. By \eqref{eq:sops-domination} and $\ell_t\leq1$,
\[
\widehat q_t(x,a)\widehat\ell_t(x,a)^2
\leq\frac{\widehat q_t(x,a)}{u_t(x,a)+\gamma}\widehat\ell_t(x,a)
\leq\frac Ld\widehat\ell_t(x,a).
\]
Applying the bound on the total estimated loss in Lemma~\ref{lem:loss-estimates} gives
\begin{equation}
R^{(2)}\leq\frac{L\log(|X|^2|A|)}\eta
+\frac{L\eta}{d}|X||A|T
+\frac{L^2\eta}{d\gamma}\log\frac{8L}{\delta}.
\label{eq:omd-final}
\end{equation}
\paragraph{Regret from baseline use}
Equation~\eqref{eq:mixture} gives $R^{(3)}\leq L\sum_t\lambda_{t-1}$. We bound the first episode by $L$. For $t\geq2$ with $\lambda_{t-1}>0$, choose an index $i_t$ attaining the maximum in \eqref{eq:sops-mixing}. Then
\begin{equation}
\begin{aligned}
\lambda_{t-1}
&=\frac{w_{t,i_t}-\alpha_{i_t}}{w_{t,i_t}-\beta_{i_t}}\\
&=\frac{(1-\lambda_{t-1})(w_{t,i_t}-\alpha_{i_t})}{\alpha_{i_t}-\beta_{i_t}}\\
&\leq\frac{1-\lambda_{t-1}}d(w_{t,i_t}-\alpha_{i_t}).
\label{eq:exact-identity}
\end{aligned}
\end{equation}
Since $[w_{t,i_t}-\alpha_{i_t}]_+\leq e_t$, Lemma~\ref{lem:main-weighted-errors} and \eqref{eq:exact-identity} yield
\[
\sum_t\lambda_{t-1}\leq1+\mathcal O\!\left(
\frac{L|X|}{d}\sqrt{|A|T\log\frac{8m|X|^2|A|T}{\delta}}+\frac{|X|^3|A|}{d}\log^2\frac{8m|X|^2|A|T}{\delta}\right),
\]
and hence
\begin{equation}
R^{(3)}\leq L+\mathcal O\!\left(
\frac{L^2|X|}{d}\sqrt{|A|T\log\frac{8m|X|^2|A|T}{\delta}}+\frac{L|X|^3|A|}{d}\log^2\frac{8m|X|^2|A|T}{\delta}\right).
\label{eq:mixing-regret}
\end{equation}
The weighting by execution probabilities follows \citet[Appendix E, Lemma E.2]{bollini2026}.
\paragraph{Policy loss-estimation error}
We decompose the policy loss-estimation error into a martingale fluctuation and a conditional bias:
\begin{equation}
R^{(4)}=\sum_t(\E_t[\widehat\ell_t]-\widehat\ell_t)^\top\widehat q_t
+\sum_t(\ell_t-\E_t[\widehat\ell_t])^\top\widehat q_t.
\label{eq:bias-decomposition}
\end{equation}
Define $h_t=\prod_{s=1}^{t-1}\ind{\inf_{r\in\cC_s}D(r\|\widetilde q_{s+1})<\infty}$, with $h_1=1$. The indicator $h_t$ is $\cF_{t-1}$-measurable and equals one on $\cE_T$ by Lemma~\ref{lem:sops-safety}. When $h_t=1$, all preceding projections attain finite minima, so the mixing rule and the definition of $u_t$ give $\widehat q_t\leq(L/d)u_t$. Hence
\[
0\leq\widehat\ell_t^\top\widehat q_t
\leq\frac Ld\sum_{x,a}\ell_t(x,a)I_t(x,a)
\leq\frac{L^2}{d}.
\]
By the tower property, $h_t(\E_t[\widehat\ell_t]-\widehat\ell_t)^\top\widehat q_t$ is a martingale difference sequence with respect to $(\cF_t)$, bounded in absolute value by $L^2/d$. Azuma--Hoeffding gives an event of probability at least $1-\delta/8$. On its intersection with $\cE_T$, we have
\begin{equation}
\sum_t(\E_t[\widehat\ell_t]-\widehat\ell_t)^\top\widehat q_t
\leq\frac{L^2}{d}\sqrt{2T\log(8/\delta)}.
\label{eq:candidate-bias-martingale}
\end{equation}
For the second sum, $\E_t[\widehat\ell_t(x,a)]=q_t(x,a)\ell_t(x,a)/(u_t(x,a)+\gamma)$. On $\cE_T$, $q_t\leq u_t$, so
\begin{align*}
\sum_t(\ell_t-\E_t[\widehat\ell_t])^\top\widehat q_t
&=\sum_{t,x,a}\frac{\widehat q_t(x,a)\ell_t(x,a)}{u_t(x,a)+\gamma}
(u_t(x,a)-q_t(x,a)+\gamma)\\
&\leq\frac Ld\sum_{t,x,a}(u_t(x,a)-q_t(x,a)+\gamma)\\
&\leq\mathcal O\!\left(\frac{L^2|X|}{d}\sqrt{|A|T\log\frac{8m|X|^2|A|T}{\delta}}+\frac{L|X|^3|A|}{d}\log^2\frac{8m|X|^2|A|T}{\delta}\right)+\frac{L\gamma}{d}|X||A|T,
\end{align*}
where the last line uses \eqref{eq:mixture-upper-deviation}. This proves
\begin{equation}\begin{aligned}
R^{(4)}\leq{}&\mathcal O\!\left(\frac{L^2|X|}{d}\sqrt{|A|T\log\frac{8m|X|^2|A|T}{\delta}}+\frac{L|X|^3|A|}{d}\log^2\frac{8m|X|^2|A|T}{\delta}\right)\\
&+\frac{L\gamma}{d}|X||A|T
+\frac{L^2}{d}\sqrt{2T\log(8/\delta)}.
\label{eq:candidate-bias-final}
\end{aligned}\end{equation}
\paragraph{Comparator loss-estimation error}
The comparator bound in Lemma~\ref{lem:loss-estimates} gives
\begin{equation}
R^{(5)}\leq\frac L\gamma\log\frac{8|X||A|}{\delta}.
\label{eq:comparator-bias-final}
\end{equation}
\paragraph{Combining the bounds}
Set $\eta=\gamma=\sqrt{\frac{L}{T|X||A|}\log\frac{8m|X|^2|A|T}{\delta}}$.
Combining \eqref{eq:transition-regret}, \eqref{eq:omd-final}, \eqref{eq:mixing-regret}, \eqref{eq:candidate-bias-final}, and \eqref{eq:comparator-bias-final} gives
\begin{align*}
\sum_{t=1}^T\ell_t^\top(q_t-q)
\leq\mathcal O\!\Biggl(&
\frac{L^2|X|}{d}\sqrt{|A|T\log\frac{8m|X|^2|A|T}{\delta}}
+\frac{L|X|^3|A|}{d}\log^2\frac{8m|X|^2|A|T}{\delta}\\
&+\frac{L\sqrt{L|X||A|}}{d}\sqrt{T\log\frac{8m|X|^2|A|T}{\delta}}
+\frac{L^2}{d}\sqrt{T\log\frac{8m|X|^2|A|T}{\delta}}\\
&+\frac{L^2}{d}\log\frac{8m|X|^2|A|T}{\delta}+L\Biggr).
\end{align*}
Using $0<d\leq L<|X|$, $|A|\geq1$, and $\log\frac{8m|X|^2|A|T}{\delta}\geq1$, we obtain, for every $T\geq1$,
\[
\sum_{t=1}^T\ell_t^\top(q_t-q)
\leq\mathcal O\!\left(
\frac{L^2|X|}{d}\sqrt{|A|T\log\frac{8m|X|^2|A|T}{\delta}}
+\frac{L|X|^3|A|}{d}\log^2\frac{8m|X|^2|A|T}{\delta}
\right).
\]

Using the second bound in Lemma~\ref{lem:weighted-occupancy} in \eqref{eq:five-terms} and the bound $\sum_t\ell_t^\top(q_t-q)\leq LT$, we also obtain
\[
\sum_{t=1}^T\ell_t^\top(q_t-q)
\leq\mathcal O\!\left(\frac{L|X|^2}{d}\sqrt{|A|T\log\frac{8m|X|^2|A|T}{\delta}}\right).
\]
A union bound over the events in Lemmas~\ref{lem:weighted-occupancy},~\ref{lem:loss-estimates}, and~\ref{lem:sops-width}, and in \eqref{eq:candidate-bias-martingale}, gives probability at least $1-\delta$. On their intersection, both regret bounds hold for every feasible $q$, and Lemma~\ref{lem:sops-safety} gives $G^\top q_t\leq\alpha$ for all $t\in[T]$. This proves the theorem.
\end{proof}

\needspace{12\baselineskip}
\section{Analysis of Baseline Search}\label{app:identification}
We prove Theorem~\ref{thm:identification} by bounding the numbers of exploration and baseline episodes in terms of the cumulative margin gap.
We consider the search in Section~\ref{sec:identification-guarantee} with confidence parameter $\delta$ and without the horizon condition $t<T$. Let $\tau$ be its termination episode and $N$ the number of exploration episodes with $\zeta_t=1$, allowing either count to be infinite. If $2d\geq L$, the search terminates immediately with $N=\tau=0$.

\subsection{Exploration Episodes}\label{app:search-process}
\paragraph{Initialization}
With the initialization in Appendix~\ref{app:estimation}, set $\widehat q_1(x,a,x')=1/(|X_k||A||X_{k+1}|)$ for $x\in X_k$ and $x'\in X_{k+1}$. This occupancy induces the uniform initial policy. Since every $q\in\Delta(\cP_0)$ has total sum $L$,
\[
(\widehat g_{0,i}-\xi_0)^\top q=-L\leq\alpha_i-L,
\qquad
(\widehat g_{0,i}+\xi_0)^\top q=L.
\]
Thus $(\widehat q_1,L)$ solves the initial optimistic problem. Write $\overline\rho_1=L$, $w_{1,i}=L$, and $\underline\rho_1=\min_i(\alpha_i-L)\leq0$. The baseline test is $2d\geq\overline\rho_1$, while $2\underline\rho_1<\overline\rho_1$. The initial baseline probability $\lambda_1$ is computed from $w_1$.
\paragraph{Observation sequence}
We use the episode filtration from Appendix~\ref{app:estimation} and keep it constant after termination. Set $\zeta_t=0$ after termination, $\tau_0=0$, and $\tau_n=\inf\{t\geq1:\sum_{s=1}^t\zeta_s\geq n\}$ for $n\geq1$, with $\inf\varnothing=\infty$. Both $\tau$ and each $\tau_n$ are stopping times. Set $\cH_n=\cF_{\tau_n\wedge\tau}$ and write $q_n=q^{P,\widehat\pi_n}$ for each policy selected by the search. The decision to continue exploration is $\cH_{n-1}$-measurable. When exploration continues, $\cH_{n-1}=\cF_{\tau_{n-1}}$, and $\widehat\pi_n$ and $\lambda_n$ are $\cH_{n-1}$-measurable.
\begin{lemma}\label{lem:candidate-sequence}
For each policy $\widehat\pi_n$ selected for exploration, $1-\lambda_n\geq d/L$ and $\tau_n<\infty$ almost surely. Conditional on $\cH_{n-1}$, episode $\tau_n$ follows the transition kernel $P$ and policy $\widehat\pi_n$, and its cost matrix has distribution $\mathcal G$. In particular, $\E[I_n(x,a)\mid\cH_{n-1}]=q_n(x,a)$ for every $(x,a)$.
\end{lemma}
\begin{proof}
The search keeps $\widehat\pi_n$ and its execution probability $1-\lambda_n$ unchanged until this policy is executed. If $\lambda_n>0$, choose an index $i_n$ attaining the maximum that defines $\lambda_n$. Then
\begin{equation}
1-\lambda_n=\frac{\alpha_{i_n}-\beta_{i_n}}{w_{n,i_n}-\beta_{i_n}}\geq\frac dL.
\label{eq:waiting-parameters}
\end{equation}
For $\lambda_n=0$, the same bound follows from $d\leq L$. For every integer $h\geq1$, the independent policy draws give $\Prob(\tau_n-\tau_{n-1}>h\mid\cH_{n-1})=\lambda_n^h\leq(1-d/L)^h$. Thus $\tau_n<\infty$ almost surely whenever the search continues with policy $\widehat\pi_n$.
For each episode $s$, the event $\{\tau_{n-1}<s\leq\tau_n\}$ is $\cF_{s-1}$-measurable. On this event, the policy draw is independent of the current cost matrix and transition randomness. Thus, conditioning on $\zeta_s=1$ selects $\widehat\pi_n$ without changing their distributions. Averaging over the number of preceding baseline episodes gives
\begin{align*}
\E[I_n(x,a)\mid\cH_{n-1}]
&=\sum_{h=0}^\infty(1-\lambda_n)\lambda_n^h q_n(x,a)
=q_n(x,a).
\end{align*}
\end{proof}

\subsection{Margin Bounds and Constraint Violation}
Recall $\Delta_n=\overline\rho_n-\underline\rho_n$ for each policy whose cost upper bounds $w_n$ are computed, including the initial policy. Let $b_{n-1}$ be defined by \eqref{eq:row-width-definition} for $\cP_{n-1}$.
\begin{lemma}\label{lem:margin-bounds}
On $\cE_{\mathrm{s}}$, the optimistic problem is feasible and $\overline\rho_n\geq\rho$ at initialization and after every estimate update. For the initial policy and every subsequent policy for which $w_n$ is computed, we have $g_i^\top q_n\leq w_{n,i}$ for every $i\in[m]$, $0\leq\Delta_n\leq4Lb_{n-1}^\top q_n+2\xi_{n-1}^\top q_n$, and $\underline\rho_n\leq\min_i(\alpha_i-g_i^\top q_n)\leq\rho$.
\end{lemma}
\begin{proof}
Let $q^\rho$ attain the maximum defining $\rho$ in Section~\ref{sec:setting}. On $\cE_{\mathrm{s}}$, it belongs to $\Delta(\cP_{n-1})$ and satisfies
\[
(\widehat G_{n-1}-\Xi_{n-1})^\top q^\rho
\leq G^\top q^\rho\leq\alpha-\rho\one.
\]
The occupancy $q^\rho$ is feasible for \eqref{eq:optimistic-margin-search}. Since $\rho\leq L$, its objective value is at least $\rho$, which gives $\overline\rho_n\geq\rho$.
The true occupancy $q_n$ belongs to $\Delta(\cP_{n-1},\widehat\pi_n)$. Its cost is bounded by both $L$ and $(\widehat g_{n-1,i}+\xi_{n-1})^\top q_n$, so the definition of $w_{n,i}$ gives $g_i^\top q_n\leq w_{n,i}$.
For any $r\in\Delta(\cP_{n-1},\widehat\pi_n)$, Lemmas~\ref{lem:realization} and~\ref{lem:fixed-policy-deviation} give
$\norm{r-q_n}_1\leq Lb_{n-1}^\top q_n$ and 
$\norm{\widehat q_n-q_n}_1\leq Lb_{n-1}^\top q_n$.

The optimistic constraint and $\norm{\widehat g_{n-1,i}-\xi_{n-1}}_\infty\leq1$ imply
\begin{align*}
(\widehat g_{n-1,i}+\xi_{n-1})^\top r
&=(\widehat g_{n-1,i}-\xi_{n-1})^\top\widehat q_n
+(\widehat g_{n-1,i}-\xi_{n-1})^\top(r-\widehat q_n)\\
&\qquad+2\xi_{n-1}^\top r\\
&\leq\alpha_i-\overline\rho_n
+\norm{r-\widehat q_n}_1+2\xi_{n-1}^\top r\\
&\leq\alpha_i-\overline\rho_n
+3\norm{r-q_n}_1+\norm{q_n-\widehat q_n}_1+2\xi_{n-1}^\top q_n\\
&\leq\alpha_i-\overline\rho_n+4Lb_{n-1}^\top q_n+2\xi_{n-1}^\top q_n.
\end{align*}
The second inequality uses the triangle inequality and $\norm{\xi_{n-1}}_\infty\leq1$. Maximizing over $r$ and using the definition of $w_{n,i}$ give the same upper bound for $w_{n,i}$. Taking the minimum over $i$ in $\alpha_i-w_{n,i}$ proves the upper bound on $\Delta_n$. The bounds on $\underline\rho_n$ follow from $g_i^\top q_n\leq w_{n,i}$ and the definition of $\rho$. Together with $\overline\rho_n\geq\rho$, they also give $\Delta_n\geq0$.
\end{proof}
Write $\mu_n=\lambda_n/(1-\lambda_n)$ for the conditional expected number of baseline episodes before policy $\widehat\pi_n$ is executed.
\begin{lemma}\label{lem:identification-safety}
On $\cE_{\mathrm{s}}$, every randomized policy used during the search satisfies $G^\top q_t\leq\alpha$. For every integer $n\in [N]$, $\mu_n\leq\Delta_n/d$.
\end{lemma}
\begin{proof}
The function $z\mapsto z/(1-z)$ is increasing on $[0,1)$. Applying it to both branches of the mixing rule gives
\begin{align*}
\mu_n
&=\max_{i\in[m]}\frac{[w_{n,i}-\alpha_i]_+}{\alpha_i-\beta_i}
\leq\frac{[-\underline\rho_n]_+}{d}\\
&=\frac{[\Delta_n-\overline\rho_n]_+}{d}
\leq\frac{\Delta_n}{d}.
\end{align*}
The last inequality follows from Lemma~\ref{lem:margin-bounds}.
During an episode with search index $n$, the randomized policy has occupancy $q_t=\lambda_nq^{P,\pi^\diamond}+(1-\lambda_n)q_n$. Fix $i\in[m]$. If $w_{n,i}\leq\alpha_i$, then $g_i^\top q_t\leq\lambda_n\beta_i+(1-\lambda_n)w_{n,i}\leq\alpha_i$. If $w_{n,i}>\alpha_i$, the mixing rule gives $\lambda_n\geq(w_{n,i}-\alpha_i)/(w_{n,i}-\beta_i)$, and hence
\begin{align*}
g_i^\top q_t
&\leq\lambda_n\beta_i+(1-\lambda_n)w_{n,i}\\
&=w_{n,i}-\lambda_n(w_{n,i}-\beta_i)\\
&\leq\alpha_i.
\end{align*}
\end{proof}
\subsection{Number of Exploration Episodes}\label{app:search-gap}
\begin{proof}[Proof of Lemma~\ref{lem:main-counts}]
The factor $\log\frac{m|X|^2|A|(n+1)^3}{\delta}$ increases with the search index. Applying \eqref{eq:row-radius-rate} and \eqref{eq:visit-count-sum} therefore gives, for any $x\in X_k$ and $a\in A$, 
\[
\sum_{j=1}^n b_{j-1}(x,a)I_j(x,a)
\leq\mathcal O\!\left(\sqrt{|X_{k+1}|N_n(x,a)\log\frac{m|X|^2|A|(n+1)^3}{\delta}}\right).
\]
Since $\sum_{x\in X_k,a}N_n(x,a)=n$, we use the Cauchy--Schwarz inequality and it yields
\begin{equation}
\begin{aligned}
\sum_{j=1}^n b_{j-1}^\top I_j
&\leq\mathcal O\!\left(\sqrt{|A|n\log\frac{m|X|^2|A|(n+1)^3}{\delta}}\sum_{k=0}^{L-1}\sqrt{|X_k||X_{k+1}|}\right)\\
&\leq\mathcal O\!\left(|X|\sqrt{|A|n\log\frac{m|X|^2|A|(n+1)^3}{\delta}}\right).
\label{eq:row-visits}
\end{aligned}
\end{equation}
The last inequality uses $\sqrt{uv}\leq(u+v)/2$. Applying \eqref{eq:cost-radius} and \eqref{eq:visit-count-sum} to the cost confidence widths gives
\begin{equation}
\sum_{j=1}^n\xi_{j-1}^\top I_j\leq6\sqrt{L|X||A|n\log\frac{m|X|^2|A|(n+1)^3}{\delta}}.
\label{eq:candidate-width-visits}
\end{equation}
\paragraph{Martingale bounds}
By Lemma~\ref{lem:candidate-sequence}, $b_{j-1}^\top(q_j-I_j)$ and $\xi_{j-1}^\top(q_j-I_j)$ have conditional mean zero given $\cH_{j-1}$ and absolute values at most $2L$ and $L$, respectively. Extend both sequences by zero after the search terminates. The decision to continue exploration is $\cH_{j-1}$-measurable, so the extension preserves the martingale difference property.
For each integer $n\geq1$, apply Azuma--Hoeffding with failure probability $\delta/[8(n+1)^2]$ to each sequence. Since $m|X|^2|A|(n+1)\geq8$, we have $\log(8(n+1)^2/\delta)\leq\log\frac{m|X|^2|A|(n+1)^3}{\delta}$. A union bound gives
\begin{align}
\sum_{j=1}^n b_{j-1}^\top q_j
&\leq\sum_{j=1}^n b_{j-1}^\top I_j+2L\sqrt{2n\log\frac{m|X|^2|A|(n+1)^3}{\delta}},\label{eq:row-martingale}\\*
\sum_{j=1}^n\xi_{j-1}^\top q_j
&\leq\sum_{j=1}^n\xi_{j-1}^\top I_j+L\sqrt{2n\log\frac{m|X|^2|A|(n+1)^3}{\delta}},
\label{eq:width-martingale}
\end{align}
for every integer $n\in [N]$, with total failure probability less than $\delta/4$.
On the intersection of this event and $\cE_{\mathrm{s}}$, the pointwise gap bound in Lemma~\ref{lem:margin-bounds} and \eqref{eq:row-visits}--\eqref{eq:width-martingale} give
\begin{equation}
\begin{aligned}
\sum_{j=1}^n\Delta_j
&\leq4L\sum_{j=1}^n b_{j-1}^\top q_j
+2\sum_{j=1}^n\xi_{j-1}^\top q_j\\
&\leq\mathcal O\!\left(L|X|\sqrt{|A|n\log\frac{m|X|^2|A|(n+1)^3}{\delta}}\right)
\end{aligned}
\label{eq:cumulative-margin-gap}
\end{equation}
for every finite integer $n\in [N]$, where we used $|X|\geq L+1$. The combined failure probability is at most $3\delta/8\leq\delta/2$, proving the lemma.
\end{proof}

\begin{lemma}\label{lem:candidate-stopping}
On the event established in the proof of Lemma~\ref{lem:main-counts}, the search terminates after
\[
N\leq\mathcal O\!\left(\frac{L^2|X|^2|A|}{\rho^2}
\log\frac{mL^2|X|^4|A|^2}{\rho^2\delta}\right)
\]
exploration episodes. The returned policy, cost bounds, and margin satisfy the guarantees in Theorem~\ref{thm:identification}.
\end{lemma}
\begin{proof}
The $j$-th exploration episode occurs only if $2\underline\rho_j<\overline\rho_j$. Lemma~\ref{lem:margin-bounds} then gives $\Delta_j=\overline\rho_j-\underline\rho_j>\rho/2$. For every finite integer $n\in [N]$, \eqref{eq:cumulative-margin-gap} implies
\begin{equation}\begin{aligned}
\frac{n\rho}{2}&<\sum_{j=1}^n\Delta_j\\
&\leq\mathcal O\!\left(L|X|\sqrt{|A|n\log\frac{m|X|^2|A|(n+1)^3}{\delta}}\right).
\label{eq:fixed-point}
\end{aligned}\end{equation}
Squaring gives $n\leq a\log(b(n+1)^3)$ for some $a=\mathcal O(L^2|X|^2|A|/\rho^2)$ with $a\geq2$, where $b=m|X|^2|A|/\delta\geq2$. The inequality $\log x\leq x-1$ gives
\begin{align*}
n
&\leq a\log\bigl(b(6a)^3\bigr)+3a\log\frac{n+1}{6a}\\
&\leq a\log\bigl(b(6a)^3\bigr)+\frac{n+1}{2}-3a.
\end{align*}
Hence $n\leq2a\log(b(6a)^3)=\mathcal O(a\log(ab))$. Since this bound holds for every finite $n\leq N$, we have $N<\infty$. Lemma~\ref{lem:candidate-sequence} ensures that each selected exploration policy is executed in finite time. The search therefore terminates, with
\begin{equation}
N\leq\mathcal O\!\left(
\frac{L^2|X|^2|A|}{\rho^2}
\log\frac{mL^2|X|^4|A|^2}{\rho^2\delta}\right).
\label{eq:explicit-candidate-count}
\end{equation}
The bound also holds for $N=0$.
For $N\geq1$, the stopping tests are evaluated after the update following episode $\tau_N$, so $\tau=\tau_N$.
\paragraph{Returned policy}
At termination, the search index is $N+1$. If $2d\geq\overline\rho_{N+1}$, then $d\geq\overline\rho_{N+1}/2\geq\rho/2$, so $(\pi^\diamond,\beta,d)$ satisfies the stated guarantee. Otherwise, the search returns $(\widehat\pi_{N+1},w_{N+1})$ after verifying $2\underline\rho_{N+1}\geq\overline\rho_{N+1}$. Thus $\underline\rho_{N+1}\geq\overline\rho_{N+1}/2\geq\rho/2$, and Lemma~\ref{lem:margin-bounds} gives $G^\top q_{N+1}\leq w_{N+1}<\alpha$ and $\underline\rho_{N+1}\leq\rho$. The optimistic problem remains feasible on $\cE_{\mathrm{s}}$, so termination occurs through one of these two tests.
\end{proof}
\subsection{Number of Baseline Episodes}\label{app:baseline-episodes}
For $n\in [N]$, let $H_n=\tau_n-\tau_{n-1}-1$ be the number of baseline episodes between consecutive exploration episodes, with $\tau_0=0$. Conditional on $\cH_{n-1}$, it satisfies $\Prob(H_n=h\mid\cH_{n-1})=(1-\lambda_n)\lambda_n^h$ for $h=0,1,\ldots$, so $\E[H_n\mid\cH_{n-1}]=\mu_n$.
\begin{lemma}\label{lem:waiting}
With probability at least $1-\delta/4$, it holds that 
$\sum_{j=1}^n H_j
\leq2\sum_{j=1}^n\mu_j+\frac{4L}{d}\log\frac4\delta$ 
for every integer $n\in [N]$.
\end{lemma}
\begin{proof}
Set $\theta=d/(4L)$ and extend $H_j$ and $\mu_j$ by zero after the search terminates. While the search continues, \eqref{eq:waiting-parameters} gives $\theta\leq(1-\lambda_j)/4$. Since $\theta\leq1/4$ and $e^\theta-1\leq4\theta/3$, we have $\mu_j(e^\theta-1)\leq1/3$. The geometric distribution therefore gives
\[
\E[e^{\theta H_j}\mid\cH_{j-1}]
=\frac{1-\lambda_j}{1-\lambda_je^\theta}
=\frac1{1-\mu_j(e^\theta-1)}.
\]
Using $-\log(1-x)\leq x/(1-x)$ for $0\leq x<1$, we obtain
\begin{align*}
\log\E[e^{\theta H_j}\mid\cH_{j-1}]
&=-\log\bigl(1-\mu_j(e^\theta-1)\bigr)\\
&\leq\frac{\mu_j(e^\theta-1)}{1-\mu_j(e^\theta-1)}\\
&\leq\frac32\mu_j(e^\theta-1)\leq2\theta\mu_j.
\end{align*}
The decision to continue exploration is $\cH_{j-1}$-measurable. Thus $M_n=\exp\{\theta\sum_{j=1}^n(H_j-2\mu_j)\}$ is a nonnegative supermartingale with $M_0=1$. Ville's inequality yields $\Prob(\sup_{n\geq0}M_n\geq4/\delta)\leq\delta/4$. Rearranging on the complementary event proves the bound.
\end{proof}

\begin{proof}[Proof of Lemma~\ref{lem:main-baseline-count}]
Equation~\eqref{eq:main-gap-mixing} gives $\mu_j\leq[-\underline\rho_j]_+/d$ for every exploration policy. On the event in Lemma~\ref{lem:waiting}, we use $\tau_n-n=\sum_{j=1}^nH_j$ and substitute this bound to obtain
\[
\tau_n-n\leq\frac2d\sum_{j=1}^n[-\underline\rho_j]_+
+\frac{4L}{d}\log\frac4\delta
\]
 for every finite integer $n\in [N]$, with probability at least $1-\delta/4$.
\end{proof}

\needspace{9\baselineskip}
\subsection{Proof of Theorem~\ref{thm:identification}}\label{app:search-episodes}
\begin{proof}[Proof of Theorem~\ref{thm:identification}]
The intersection of the events in the proof of Lemma~\ref{lem:main-counts} and in Lemma~\ref{lem:waiting} has failure probability at most $5\delta/8$. On this event, Lemma~\ref{lem:candidate-stopping} gives termination, the bound on $N$, and the guarantee for the returned policy. Lemma~\ref{lem:identification-safety} gives $G^\top q_t\leq\alpha$ for every $t\leq\tau$.
It remains to bound $\tau$. If $N=0$, then $\tau=0$. Otherwise, $\tau=\tau_N$.
Equation~\eqref{eq:explicit-candidate-count} gives $\log\frac{m|X|^2|A|(N+1)^3}{\delta}=\mathcal O(\log(mL^2|X|^4|A|^2/(\rho^2\delta)))$. Substituting \eqref{eq:explicit-candidate-count} into \eqref{eq:cumulative-margin-gap} therefore yields
\[
\sum_{j=1}^N\Delta_j
\leq\mathcal O\!\left(\frac{L^2|X|^2|A|}{\rho}
\log\frac{mL^2|X|^4|A|^2}{\rho^2\delta}\right).
\]
Since $\tau=N+\sum_{j=1}^N H_j$, \eqref{eq:explicit-candidate-count} and Lemmas~\ref{lem:waiting} and~\ref{lem:identification-safety} give
\begin{equation}
\begin{aligned}
\tau
&\leq N+\frac2d\sum_{j=1}^N\Delta_j
+\frac{4L}{d}\log\frac4\delta\\
&\leq\mathcal O\!\left(\frac{L^2|X|^2|A|}{d\rho}
\log\frac{mL^2|X|^4|A|^2}{\rho^2\delta}\right),
\label{eq:explicit-episode-count}
\end{aligned}
\end{equation}
where the last inequality uses $d\leq\rho\leq L$.
Together with the bounds for the returned policy and $5\delta/8<\delta$, this proves the theorem.
\end{proof}

\section{Proof of Theorem~\ref{thm:combined}}\label{app:combined}
\begin{proof}[Proof of Theorem~\ref{thm:combined}]
Let $\tau$ be the termination episode of the search defined in Appendix~\ref{app:identification}. If needed, continue the search beyond $T$ using independent samples from the same model and zero losses. This agrees with Algorithm~\ref{alg:identification} through episode $\tau\wedge T$.
Apply Theorem~\ref{thm:identification} with confidence parameter $\delta/2$ and let $\cE$ be the event on which its guarantees hold. On $\cE$, every randomized policy used during the search satisfies the constraints, \eqref{eq:explicit-episode-count} bounds $\tau$, and the returned cost upper bounds give $\widehat\rho:=\min_{i\in[m]}(\alpha_i-\widehat\beta_i)\geq\rho/2$.
\paragraph{Regret after the search}
We use the episode filtration from Appendix~\ref{app:estimation}. At the bounded stopping time $\tau\wedge T$, define
\[
\mathcal V=\{\tau<T,\ G^\top q^{P,\widehat\pi^\diamond}\leq\widehat\beta<\alpha\}
\in\cF_{\tau\wedge T}.
\]
On $\mathcal V$, the returned baseline and cost bounds satisfy Assumption~\ref{ass:baseline}. These inputs and the remaining horizon $T'=T-\tau$ are $\cF_{\tau\wedge T}$-measurable. Conditional on $\cF_{\tau\wedge T}$, the remaining episodes satisfy the observation model in Section~\ref{sec:setting}, with non-anticipating losses. Algorithm~\ref{alg:identification} initializes S-OPS with zero observation counts, confidence parameter $\delta/2$, and learning rates set for $T'$.
Let $\mathcal B$ be the event that S-OPS is restarted and the common high-probability event established in Appendix~\ref{app:sops-regret} fails. Applying Theorem~\ref{thm:sops} conditionally on $\cF_{\tau\wedge T}$ gives
\begin{equation}
\ind{\mathcal V}\Prob(\mathcal B\mid\cF_{\tau\wedge T})
\leq\frac\delta2\ind{\mathcal V}.
\label{eq:conditional-restart}
\end{equation}
On $\mathcal V\cap\mathcal B^c$, the regret bound holds for all feasible occupancy measures, so it applies to $q^\star$ in \eqref{eq:regret}. On $\cE\cap\{\tau<T\}\cap\mathcal B^c$, $G^\top q_t\leq\alpha$ for every $\tau<t\leq T$ and
\begin{equation}
\begin{aligned}
\sum_{t=\tau+1}^T\ell_t^\top(q_t-q^\star)
&\leq\mathcal O\!\left(\frac{L^2|X|}{\widehat\rho}
\sqrt{|A|T'\log\frac{16m|X|^2|A|T'}{\delta}}
+\frac{L|X|^3|A|}{\widehat\rho}\log^2\frac{16m|X|^2|A|T'}{\delta}\right)\\
&\leq\mathcal O\!\left(\frac{L^2|X|}{\rho}
\sqrt{|A|T\log\frac{16m|X|^2|A|T}{\delta}}
+\frac{L|X|^3|A|}{\rho}\log^2\frac{16m|X|^2|A|T}{\delta}\right).
\label{eq:second-phase}
\end{aligned}
\end{equation}
The last inequality uses $\widehat\rho\geq\rho/2$ and $T'\leq T$.
\paragraph{Regret during the search}
The regret against $q^\star$ in each episode is at most $L$, since losses lie in $[0,1]$ and every occupancy measure has total sum $L$. On $\cE$, the episode bound \eqref{eq:explicit-episode-count} with confidence parameter $\delta/2$ gives
\begin{equation}
\begin{aligned}
\sum_{t=1}^{\tau\wedge T}\ell_t^\top(q_t-q^\star)
&\leq L(\tau\wedge T)\leq L\tau\\
&\leq\mathcal O\!\left(\frac{L^3|X|^2|A|}{d\rho}
\log\frac{mL^2|X|^4|A|^2}{\rho^2\delta}\right).
\label{eq:identification-regret}
\end{aligned}
\end{equation}
\paragraph{Combining the phases}
For every $T\geq1$, the inequalities $(a+b)^2\leq2a^2+2b^2$ and $(\log z)^2\leq3\sqrt z$ for $z\geq1$ give
\[
|X|^3|A|\log^2\frac{16m|X|^2|A|T}{\delta}
\leq2|X|^3|A|\log^2\frac{16m|X|^6|A|^2}{\delta}
+6|X|\sqrt{|A|T}.
\]
When $\tau<T$, we combine this estimate with \eqref{eq:second-phase} and \eqref{eq:identification-regret} and use $d\leq\rho\leq L<|X|$ to obtain
\begin{equation}
R_T\leq\mathcal O\!\left(
\frac{L^2|X|}{\rho}
\sqrt{|A|T\log\frac{16m|X|^2|A|T}{\delta}}
+\frac{L^2|X|^3|A|}{d\rho}
\log^2\frac{16mL^2|X|^6|A|^2}{\rho^2\delta}\right).
\label{eq:combined-explicit}
\end{equation}
If $\tau\geq T$, \eqref{eq:identification-regret} gives the same bound. The tower property and \eqref{eq:conditional-restart} give
\[
\Prob\bigl(\cE^c\cup(\mathcal V\cap\mathcal B)\bigr)
\leq\Prob\bigl(\cE^c\bigr)
+\E\!\left[\ind{\mathcal V}\Prob(\mathcal B\mid\cF_{\tau\wedge T})\right]
\leq\delta.
\]
On the complementary event, the search succeeds and S-OPS satisfies its guarantee whenever it is invoked. Thus, with probability at least $1-\delta$, $G^\top q_t\leq\alpha$ for every $t\in[T]$ and the stated regret bound holds.
\end{proof}

\needspace{12\baselineskip}
\section{Proof of Theorem~\ref{thm:lower}}\label{app:lower}
We use two pairs of instances to prove the $\sqrt T/\rho$ and $1/(d\rho)$ terms separately. Both pairs have initial margin $d$ and Slater margin $\rho$.
\begin{proof}[Proof of Theorem~\ref{thm:lower}]
Fix an algorithm satisfying the safety requirement of Theorem~\ref{thm:lower}.
In both constructions, we use $L=m=1$, actions $a_0,a_1,a_2$, and threshold $\alpha=1/2$.
The baseline policy always chooses $a_0$, whose deterministic cost is $\beta=1/2-d$.
The costs of $a_1$ and $a_2$ are Bernoulli variables, independent across actions and episodes. The losses are deterministic and constant across episodes.
The instances in each pair differ only in their cost distributions. All vectors are indexed by $(a_0,a_1,a_2)$.\par
In episode $t$, let $q_t$ be the action distribution selected by the algorithm, $A_t\sim q_t$ the sampled action, and $Y_t$ its observed cost. For each pair $M_1,M_2$, let $\mathbb P^j$ and $\mathbb E^j$ denote the law of $(q_t,A_t,Y_t)_{t=1}^T$ and the corresponding expectation in $M_j$.

\paragraph{The learning term}
Set $\epsilon=1/(8\sqrt T)$ and use the constant loss vector $\ell=(1/2,0,1/2)$. The horizon condition and $d\leq\rho$ give $T\geq1/\rho^2$, so $\epsilon\leq\rho/8$. Consider instances $M_1,M_2$ with mean cost vectors
\[\begin{aligned}
g^1&=(1/2-d,\,1/2+\epsilon,\,1/2-\rho),\\
g^2&=(1/2-d,\,1/2,\,1/2-\rho).
\end{aligned}\]
For $j=1,2$, the Slater margin is $1/2-\min_a g^j(a)=\rho$, and the baseline margin is $d$. In $M_1$,
\begin{align*}
(g^1)^\top q_t-1/2
&=\epsilon-(\epsilon+d)q_t(a_0)-(\epsilon+\rho)q_t(a_2)\\
&\geq\epsilon-(\epsilon+\rho)(1-q_t(a_1)).
\end{align*}
Let $E=\{\sum_{t=1}^T(1-q_t(a_1))\geq T\epsilon/(\epsilon+\rho)\}$. Zero constraint violation in $M_1$ requires $1-q_t(a_1)\geq\epsilon/(\epsilon+\rho)$ in every episode, so $\mathbb P^1(E)\geq1-\delta$. Only the cost distribution of $a_1$ differs between the two instances. The chain rule for KL divergence \citep{lattimore2020,pacchiano2021}, applied to the conditional distributions of policy choices, actions, and observed costs, gives
\begin{equation}
\begin{aligned}
\KL(\mathbb P^1\Vert\mathbb P^2)
&=\mathbb E^1\left[\sum_{t=1}^Tq_t(a_1)\right]\kl(1/2+\epsilon,1/2)\\
&\leq4T\epsilon^2=1/16.
\label{eq:learning-kl}
\end{aligned}
\end{equation}
The inequality uses $\kl(r,s)\leq(r-s)^2/(s(1-s))$ for $r,s\in(0,1)$. Pinsker's inequality then gives $\TV(\mathbb P^1,\mathbb P^2)<1/4$ and $\mathbb P^2(E)\geq3/4-\delta$. In $M_2$, action $a_1$ is feasible and has zero loss. On $E$,
\begin{equation}
\begin{aligned}
R_T&=\frac12\sum_{t=1}^T(1-q_t(a_1))
\geq\frac{T\epsilon}{2(\epsilon+\rho)}
\geq\frac{T\epsilon}{4\rho}
=\frac{\sqrt T}{32\rho}.
\label{eq:learning-lower}
\end{aligned}
\end{equation}
This construction extends the comparison in \citet[Appendix F]{stradi2025} by prescribing the initial baseline margin separately from the Slater margin.
\paragraph{The exploration term}
Suppose $d\leq\rho/4$. Take $\ell=(1,0,0)$ and
\[\begin{aligned}
g^1&=(1/2-d,\,1/2-\rho,\,1/2+2\rho),\\
g^2&=(1/2-d,\,1/2+2\rho,\,1/2-\rho).
\end{aligned}\]
Again, $1/2-\min_a g^j(a)=\rho$ and the baseline margin is $d$. Each instance has a feasible zero-loss action. Set $p_t=q_t(a_1)+q_t(a_2)$ and $p=2d/(\rho+2d)$. Since
\[
(g^1)^\top q_t+(g^2)^\top q_t=1-2d+(\rho+2d)p_t,
\]
a distribution feasible in both instances satisfies
\begin{equation}
p_t\leq p\leq\frac13.
\label{eq:common-feasibility}
\end{equation}
For $t\in[T]$, let $\cF_{t-1}=\sigma(q_1,(A_s,Y_s,q_{s+1})_{1\leq s<t})$ be the history before sampling $A_t$, and set $\cF_T=\cF_{T-1}\vee\sigma(A_T,Y_T)$. Let $n=\lceil1/(768d\rho)\rceil$ and let $\sigma$ be the first episode in which the selected policy is infeasible in either instance, capped at $n$:
\[
\sigma=\min\bigl(\{t\in[n]:\max_{j\in\{1,2\}}(g^j)^\top q_t>1/2\}\cup\{n\}\bigr).
\]
For every $t<\sigma$, the policy is feasible in both instances, so \eqref{eq:common-feasibility} gives $p_t\leq p$. Since $d\rho\leq1/256$, we have $n\leq1+1/(768d\rho)\leq1/(d\rho)\leq T$. The policy $q_t$ is $\cF_{t-1}$-measurable, so $\sigma-1$ is a bounded stopping time. Let $\mathbb Q^j=\mathbb P^j|_{\cF_{\sigma-1}}$ for $j=1,2$. This $\sigma$-algebra contains the policy choices through $q_\sigma$ and the actions and observed costs from episodes $t<\sigma$.
For each action $a$, let $\nu_{j,a}$ be its cost distribution in $M_j$. The conditional policy and action distributions are the same in both instances, so their contributions to the likelihood ratio cancel. The common stopping rule therefore gives
\begin{equation}
\log\frac{d\mathbb Q^1}{d\mathbb Q^2}
=\sum_{t=1}^{n-1}\ind{t<\sigma}
\log\frac{d\nu_{1,A_t}}{d\nu_{2,A_t}}(Y_t).
\label{eq:stopped-likelihood}
\end{equation}
For $a_1$ and $a_2$, the mean costs differ by $3\rho$ between instances and lie in $[3/8,3/4]$. Since $s(1-s)\geq3/16$ on this interval, their Bernoulli KL divergences are at most $48\rho^2$ in either direction. The KL divergence for the baseline action is zero. Taking expectations in \eqref{eq:stopped-likelihood} and conditioning on $\cF_{t-1}$ gives
\begin{equation}
\begin{aligned}
\KL(\mathbb Q^1\Vert\mathbb Q^2)
&=\mathbb E^1\left[
\sum_{t=1}^{n-1}\ind{t<\sigma}\sum_aq_t(a)\KL(\nu_{1,a}\Vert\nu_{2,a})\right]\\
&\leq48\rho^2\mathbb E^1
\left[\sum_{t=1}^{n-1}\ind{t<\sigma}p_t\right]\\
&\leq48(n-1)p\rho^2
\leq96(n-1)d\rho\leq1/8.
\label{eq:stopped-kl}
\end{aligned}
\end{equation}
The last two lines use $p_t\leq p$ before $\sigma$, $p\leq2d/\rho$, and $n-1<1/(768d\rho)$. Pinsker's inequality therefore gives $\TV(\mathbb Q^1,\mathbb Q^2)\leq1/4$.
Define $E_j=\{(g^j)^\top q_\sigma>1/2\}$ for $j=1,2$. Each event $E_j$ belongs to $\cF_{\sigma-1}$ and implies a constraint violation in $M_j$. Hence $\mathbb P^j(E_j)\leq\delta$. The event $E_1\cup E_2$ occurs precisely when a policy in the first $n$ episodes violates a constraint in either instance. Consequently,
\begin{equation}
\begin{aligned}
\mathbb P^1(E_1\cup E_2)
&\leq\mathbb P^1(E_1)+\mathbb P^1(E_2)\\
&\leq\delta+\mathbb P^2(E_2)+\TV(\mathbb Q^1,\mathbb Q^2)\\
&\leq2\delta+1/4.
\label{eq:stopping-probability}
\end{aligned}
\end{equation}
Thus $\mathbb P^1((E_1\cup E_2)^c)\geq3/4-2\delta$. On this event, $\sigma=n$ and every policy through episode $n$ is feasible in both instances. Therefore $p_t\leq p\leq1/3$ for every $t\leq n$, and
\begin{equation}
\begin{aligned}
R_T&=\sum_{t=1}^T(1-p_t)\geq\sum_{t=1}^n(1-p_t)
\geq\frac{2n}{3}\geq\frac1{1152d\rho}.
\label{eq:exploration-lower}
\end{aligned}
\end{equation}
\paragraph{Combining the two terms}
If $d\sqrt T\geq1/2$, then $1/(d\rho)\leq2\sqrt T/\rho$, so \eqref{eq:learning-lower} gives $R_T\geq(\sqrt T/\rho+1/(d\rho))/96$ in instance $M_2$ of the first construction.

Otherwise, $T\geq1/(d\rho)$ implies $d/\rho\leq d^2T<1/4$, and $\sqrt T/\rho<1/(2d\rho)$. The second construction therefore applies, and \eqref{eq:exploration-lower} gives $R_T\geq(\sqrt T/\rho+1/(d\rho))/1728$ in its instance $M_1$. Both cases establish the stated lower bound with probability at least $3/4-2\delta$.
\end{proof}

\end{document}